\documentclass{article}

\usepackage[accepted]{icml2019}
\usepackage{amsmath}
\usepackage{amssymb}
\usepackage{nccmath}
\usepackage{mathtools}
\usepackage{graphicx} 
\usepackage{hyperref}

\usepackage{amsthm}
\usepackage{enumitem}
\usepackage{mwe}
\usepackage{subfig}

\newtheorem{theorem}{Theorem}
\newcommand{\SLB}{\textsc{\small{SLB}}\xspace}
\newcommand{\alg}{\textsc{\small{PSLB}}\xspace}
\newcommand{\oful}{\textsc{\small{OFUL}}\xspace}
\newcommand{\OFUL}{\textsc{\small{OFUL}}\xspace}
\newcommand{\OFU}{\textsc{\small{OFU}}\xspace}
\usepackage{amsmath,amsfonts,bm}
\usepackage{xspace}
\newcommand{\wt}[1]{\widetilde{#1}}

\newcommand{\OO}{\mathcal{O}}

\newtheorem{lemma}[theorem]{Lemma}

\newtheorem{assumption}{Assumption}
\DeclareMathOperator{\tr}{tr}
\DeclareMathOperator{\spn}{span}

\def\eqref#1{equation~\ref{#1}}

\def\1{\bm{1}}

\DeclareMathAlphabet{\mathsfit}{\encodingdefault}{\sfdefault}{m}{sl}
\SetMathAlphabet{\mathsfit}{bold}{\encodingdefault}{\sfdefault}{bx}{n}

\DeclareMathOperator*{\argmax}{arg\,max}
\DeclareMathOperator*{\argmin}{arg\,min}

\begin{document}

\twocolumn[
\icmltitle{Stochastic Linear Bandits with Hidden Low Rank Structure}

\begin{icmlauthorlist}
\icmlauthor{Sahin Lale}{Caltech}
\icmlauthor{Kamyar Azizzadenesheli}{UCI}
\icmlauthor{Anima Anandkumar}{Caltech}
\icmlauthor{Babak Hassibi}{Caltech}
\end{icmlauthorlist}

\icmlaffiliation{Caltech}{California Institute of Technology, Pasadena, CA, USA}
\icmlaffiliation{UCI}{University of California, Irvine, CA,
USA}

\icmlcorrespondingauthor{Sahin Lale}{alale@caltech.edu}

\icmlkeywords{Linear Bandits, Dimensionality Reduction, Deep Learning}

\vskip 0.3in
]
\printAffiliationsAndNotice{}
\begin{abstract}
High-dimensional representations often have a lower dimensional underlying structure. This is particularly the case in many decision making settings. For example, when the representation of actions is generated from a deep neural network, it is reasonable to expect a low-rank structure whereas conventional structures like sparsity are not valid anymore. Subspace recovery methods, such as Principle Component Analysis (PCA) can find the underlying low-rank structures in the feature space and reduce the complexity of the learning tasks. In this work, we propose Projected Stochastic Linear Bandit (\alg), an algorithm for high dimensional stochastic linear bandits (\SLB) when the representation of actions has an underlying low-dimensional subspace structure. \alg deploys PCA based projection to iteratively find the low rank structure in \SLB{}s. We show that deploying projection methods assures dimensionality reduction and results in a tighter regret upper bound that is in terms of the dimensionality of the subspace and its properties, rather than the dimensionality of the ambient space. We modify the image classification task into the \SLB setting and empirically show that, when a pre-trained DNN provides the high dimensional feature representations, deploying \alg results in significant reduction of regret and faster convergence to an accurate model compared to state-of-art algorithm.


\end{abstract}

\section{INTRODUCTION}
Stochastic linear bandit (\SLB) is a class of sequential decision-making under uncertainty where an agent sequentially chooses actions from very large action sets. At each round, the agent applies its action, and as a response, the environment emits a stochastic reward whose expected value is an unknown linear function of the action. The agent's goal is to collect as much reward as possible over the course of $T$ interactions. 

In \SLB, the actions are represented as $d$-dimensional vectors, and the agent maintains limited information about the unknown linear function of reward. Through the course of interaction, the agent implicitly or explicitly constructs the model of the environment. It dedicates the decisions to not only maximize the current reward but also explore other actions to build a better estimation of the unknown linear function and guarantee higher future rewards. This is known as the \textit{exploration} vs. \textit{exploitation} trade-off.

The lack of oracle knowledge of the true environment model causes the agent to make mistakes by picking sub-optimal actions during the exploration. While the agent examines actions in the decision set, its committed mistakes accumulate. The aim of the agent is to design a strategy to minimize the cumulative cost of these mistakes, known as regret. One promising approach to minimize the regret is through utilizing the \textit{optimism in the face of uncertainty} (\OFU) principle first proposed by \citet{lai1985asymptotically}. \OFU based algorithms estimate the environment model up to its confidence interval and construct a plausible set of models within that interval. Among those models in the plausible set, they choose the most optimistic one and follow the optimal behavior suggested by the selected model for the next round of decision making. 

For general \SLB problems, \citet{abbasi2011improved} deploy the \OFU principle, propose \OFUL algorithm, and for $d$-dimensional \SLB, derive a regret upper bound  of $\wt\OO\left(d\sqrt{T}\right)$ which matches the lower bound up to a log factor. These regret bounds in high dimensional problems especially when $d$ and $T$ are about the same order are not practically tolerable. 
Fortunately, real-world problems usually are not arbitrary and may contain hidden low-dimensional structures. For example in classical recommendation systems, each item is represented by a large and highly detailed hand-engineered feature vector; therefore $d$ is intractably large. In these problems, not all the features are helpful for the recommendation task. For instance, the height of goods such as a pen is not a relevant feature for its recommendation while this feature is valuable for furnitures. Therefore the true underlying linear function in \SLB{}s is highly sparse. \citet{abbasi2012online} show how to exploit this additional structure and design a practical algorithm with regret of $\wt\OO\left(\sqrt{sdT}\right)$ where $s$ is the sparsity level of the true underlying linear function. Under slightly stronger assumptions, \citet{carpentier2012bandit} show the theory of compressed sensing can provide a tighter bound of $\wt\OO\left(s\sqrt{T}\right)$.

The contemporary success of Deep Neural Networks~(DNN) in representation learning enables classical machine learning methods to provide significant advancements in many machine learning problems, e.g., classification and  regression tasks~\citep{lecun1998gradient}. DNNs convolve the raw features of the input and construct new feature representations which replace the hand-engineered feature vectors in many real-world sequential decision making applications, e.g., recommendation systems. However, when a DNN provides the feature representations, one cannot see a sparse structure.

Dimension reduction and subspace recovery form the core of unsupervised learning methods and principal component analysis (PCA) is the main technique for linear  dimension reduction~\citep{pearson1901liii,eckart1936approximation}. At each round of \SLB, the agent chooses an action and receives the reward corresponding to that action. Therefore, the chosen action is assigned a supervised reward signal while other actions in the decision set remain unsupervised. Even though the primary motivation in the \SLB framework is decision-making within a large and stochastic decision set, the majority of prior works do not exploit possible hidden structures in these sets.
For example,~\citet{abbasi2011improved} only utilizes supervised actions, the actions selected by the algorithm, to construct the environment model. It ignores all other unsupervised actions in the decision set. On the contrary, large number of actions in the decision sets can be useful in reducing the dimension of the problem and simplifying the learning problem.

\noindent{\bf Contributions:~} In this paper, we deploy unsupervised subspace recovery using PCA to exploit the massive number of unsupervised actions which are observed in the decision sets of \SLB and reduce the dimensionality and the complexity of \SLB{}s. We propose \alg for \SLB{s} and show that if there exists an $m$-dimensional subspace structure such that the actions live in a perturbed region around this subspace, deploying \alg improves the regret upper bound to  $\textit{min}\Big\lbrace\wt\OO\left(\Upsilon\sqrt{T}\right),\wt\OO\left(d\sqrt{T}\right)\Big\rbrace$. Here $\Upsilon$ represents the difficulty of subspace recovery as a function of the structure of the problem. If learning the subspace is hard, \textit{e.g.}, the eigengap is small to analyze in a reasonable amount of samples, actions are widely distributed in the orthogonal dimensions of the subspace due to perturbation or $m\approx d$, then using projection approaches are not remedial. On the other hand, if underlying subspace is identifiable, \textit{i.e.}, large number of actions are available from the decision sets in each round, the eigengap is significant or $m \ll d$, then using subspace recovery provides faster learning of the underlying linear function; thus, smaller regret.  

We adapt the image classification tasks on MNIST~\citep{lecun1998gradient}, CIFAR-10~\citep{krizhevsky2009learning} and ImageNet~\citep{krizhevsky2012imagenet}  datasets to the \SLB framework and apply both \alg and \OFUL on these datasets. We observe that there exists a low dimensional subspace in the feature space when a pre-trained DNN produces the $d$-dimensional feature vectors. We empirically show that using subspace recovery \alg learns the underlying model significantly faster than \OFUL and provides orders of magnitude smaller regret in \SLB{}s obtained from MNIST, CIFAR-10, and ImageNet datasets. 


\section{Preliminaries}\label{Model}
For any positive integer $n$, $[n]$ denotes the set $\{1,\ldots,n\}$. The Euclidean norm of a vector $x$ is denoted by $\|x\|_2$. The spectral norm of matrix $A$ is denoted by $\| A \|_2$, \textit{ie.}, $\| A \|_2 \coloneqq \sup\{\| Ax \| : \| x \|_2 = 1\} $. $A^{\dagger}$ denotes the Moore-Penrose inverse of matrix $A$. For any symmetric and positive semi-definite matrix $M$, let $\| x \|_M$ denote the norm of a vector $x$ defined as $\| x \|_M \coloneqq \sqrt{x^TMx}$. The $j$-th eigenvalue of a symmetric matrix $A$ is denoted by $\lambda_j(A)$, where $\lambda_1(A) \geq \lambda_2(A) \geq \ldots$. The largest and smallest eigenvalue of $A$ are denoted as $\lambda_{max}(A)$ and $\lambda_{min}(A)$, respectively. $I_d$ denotes $d \times d$ identity matrix. If $Y_i$ is a column vector then $\mathbf{Y}_t$ is a matrix whose columns are $Y_1, \ldots, Y_t$ whereas if $y_i$ is a scalar then $\mathbf{y}_t$ is a column vector whose elements are $y_1, \ldots, y_t$. $\uplus_{i=1}^t D_i$ defines the multiset summation operation over the sets $D_1, \ldots, D_t$.

\paragraph{Model:~}Let $T$ be the total number of rounds. At each round $t \in [T]$, the agent is given a decision set $D_t$ with $K$ actions, $\hat{x}_{t,1}, \ldots, \hat{x}_{t,K} \in \mathbb{R}^d$. Let $V$ be an $d \times m$ orthonormal matrix with $m \leq d$, where $\spn(V)$ defines a $m$-dimensional subspace in $\mathbb{R}^d$. Consider a zero mean
true action vector, $x_{t,i} \in \mathbb{R}^d$, such that $x_{t,i} \in \spn(V)$ for all $i \in [K]$ and $t \in [T]$. Let $\psi_{t,i} \in \mathbb{R}^d$ be zero mean random vectors which are uncorrelated with true action vectors, \textit{i.e.},  $\mathbb{E}[x_{t,i}\psi_{t,i}^T]= 0$ for all $i \in [K]$ and $t \in [T]$. Each action $\hat{x}_{t,i}$ is generated as follows,
\begin{equation} \label{model}
\hat{x}_{t,i} = x_{t,i} + \psi_{t,i}.
\end{equation}
This model states that each $\hat{x}_{t,i}$ in $D_t$ is a perturbed version of the true underlying $x_{t,i}$. 
Denote the covariance matrix of $x_{t,i}$ by $\Sigma_x$. Notice that $\Sigma_x$ is rank-$m$. Perturbation vectors, $\psi_{t,i}$, are assumed to be isotropic, thus covariance matrix $\Sigma_{\psi} = \sigma^2 I_d$. Let $\lambda_{+} \coloneqq \lambda_1(\Sigma_x)$ and $\lambda_{-} \coloneqq \lambda_m(\Sigma_x)$. We will make a boundedness assumption on $x_{t,i}$ and $\psi_{t,i}$.
\begin{assumption} [Bounded Action and Perturbation Vectors] \label{bounded}
There exists finite constants, $d_x$ and $d_{\psi}$, such that for all $t \in [T]$ and $i \in [K]$ , 
\begin{equation*}
\| x_{t,i} \|_2^2 \leq d_x \lambda_{+}, \quad \|\psi_{t,i}\|_2^2 \leq d_{\psi} \sigma^2. \vspace{-.5\baselineskip}
\end{equation*}
\end{assumption}
Both $d_x$ and $d_{\psi}$ can be dependent on $m$ or $d$ and they can be interpreted as the effective dimensions of the corresponding vectors. 

At each round $t$, the agent chooses an action, $\hat{X}_t \in D_t$ and observes a reward $r_t$ such that 
\begin{equation}
r_t = (P\hat{X}_t)^T \theta_* + \eta_t \qquad \forall t \in [T]
\end{equation}
where $P = VV^T$ is the projection matrix for the $m$-dimensional subspace $\spn(V)$, $\theta_* \in \spn(V)$ is the unknown parameter vector and $\eta_t$ is the random noise at round $t$. Notice that since $\theta_* \in \spn(V)$, $(P \hat{X}_t)^T \theta_* = \hat{X}_t^T P \theta_* = \hat{X}_t^T \theta_* $ therefore, $r_t = \hat{X}_t^T \theta_* + \eta_t$.\footnote{The reward generative model of $r_t = \hat{X}_t^T \theta_* + \eta_t$ is equivalent to ${r}_t = X_t^T\theta_* + \tilde{\eta}_t$ where $\tilde{\eta}_t$ contains the randomness in $\eta_t$ as well as the perturbations due to $\psi_{t,i}$.} Consider $\{ F_t\}^{\infty}_{t=0}$ as any filtration of $\sigma$-algebras such that for any $t \geq 1$, $\hat{X}_t$ is $F_{t-1}$ measurable and $\eta_t$ is $F_t$ measurable. 
\begin{assumption}[Subgaussian Noise] \label{As_Noise}
For all $t \in [T]$, $\eta_t$ is conditionally R-sub-Gaussian where $R \geq 0$ is a fixed constant, \textit{ie.} $\forall \lambda \in \mathbb{R}, \enskip \mathbb{E}[e^{\lambda \eta_t} | F_{t-1}] \leq e^{\frac{\lambda^2 R^2}{2}}$.
\end{assumption}
This implies that $\mathbb{E}\big[\hat{X}_t^T \theta_* + \eta_t | \mathbf{\hat{X}}_t , \mathbf{\eta}_{t-1} \big] = \hat{X}_t^T \theta_*$ or equivalently $\mathbb{E}[\eta_t | F_{t-1} ] = 0$. The goal of the agent is to maximize the total expected reward accumulated in $T$ rounds, $\sum_{t=1}^T \hat{X}_t^T\theta_*$. The oracle's strategy with the knowledge of $\theta_*$ at each round $t$ is $\hat{X}^*_t = \argmax_{x\in D_t} x^T \theta_*$. We evaluate the agent's performance against the oracle performance. Define \textit{regret} as the difference between expected reward of the oracle and the agent, 
\begin{equation}
R_T \coloneqq  \sum_{t=1}^T \hat{X}_t^{*T}\theta_* - \sum_{t=1}^T \hat{X}_t^T\theta_* = \sum_{t=1}^T (X^*_t - \hat{X}_t)^T \theta_* .  
\end{equation}
The agent aims to minimize this quantity over time. In the setting described above, the agent is assumed to know that there exists a $m$-dimensional subspace of $\mathbb{R}^d$ in which true action vectors and the unknown parameter vector lie. Finally, we define some quantities about the structure of the problem for all $\delta \in (0,1)$:
\begin{align}
&g_x = \frac{\lambda_+}{\lambda_{-}}, \quad g_{\psi} = \frac{\sigma^2}{\lambda_{-}}, \quad \Gamma = 2g_{\psi} + 4\sqrt{g_x g_{\psi}} \label{quantities} \\
&\alpha \medop{=} \max(d_x, d_{\psi}), \enskip n_{\delta} \medop{=} 4\alpha\bigg(\Gamma \sqrt{\log \frac{2d}{\delta}} \medop{+} \sqrt{2g_x \log\frac{m}{\delta}}\bigg)^2 \nonumber
\end{align}

\section{Overview of \alg} \label{PSLB_Alg}
We propose \alg, a \SLB algorithm which employs subspace recovery to extract information from the unsupervised data accumulated in the \SLB. During the course of interaction, the agent constructs the confidence set of the underlying model with and without subspace recovery
, then takes the intersection of these two sets. Among the plausible models in this set, the agent deploys \OFU principle and follows the optimal action of the most optimistic model. The pseudocode of \alg is given in Algorithm ~\ref{alg:pslb}. \alg consists of 4 key elements: warm-up, subspace estimation, creating confidence sets and acting optimistically. In the following, we will discuss each of them briefly.
\begin{algorithm}[tb]
\caption{\alg}
\begin{algorithmic}[1]
   \STATE {\bfseries Input:} m, $\lambda_+$, $\lambda_-$, $\sigma^2$, $\alpha$, $\delta$
    
 \FOR{t = 1 to $T$}
 	\STATE Compute PCA over $\uplus_{i=1}^t D_i$ 
    \STATE Create $\hat{P}_t$ with first m eigenvectors
    \STATE Construct $\mathcal{C}_{p,t}$, high probability confidence set on $\hat{P}_t$
   	\STATE Construct $\mathcal{C}_{m,t}$, high probability confidence set for $\theta_*$ using  subspace recovery
    \STATE Construct $\mathcal{C}_{d,t}$, high probability confidence set for $\theta_*$ without using subspace recovery
    \STATE Construct $\mathcal{C}_{t} = \mathcal{C}_{m,t} \cap \mathcal{C}_{d,t}$ 
    \STATE $ (\tilde{P}_t, \hat{X}_t,  \tilde{\theta}_t) = \argmax_{(P^\prime, x,\theta)\in \mathcal{C}_{p,t} \times D_t \times \mathcal{C}_{t}} (P^\prime x)^T \theta $
    \STATE Play $\hat{X}_t$ and observe $r_t$
\ENDFOR
\end{algorithmic}
\label{alg:pslb}
\end{algorithm}

\subsection{Warm-Up} \label{WarmUp}

The decision set at each round $i$, $D_i$, has a finite number of actions. The algorithm needs to acquire enough samples of action vectors to reliably estimate the hidden subspace. The process of acquiring sufficient samples is considered as the warm-up period. The duration of the warm-up period, $t_{w,\delta}$, can be chosen in many ways. We set $t_{w,\delta} = \frac{n_{\delta}}{K}$ based on the theoretical analysis outlined in Section \ref{ProjErrAnaly}. The crux of this choice is to provide a theoretical guarantee of convergence to the underlying subspace. In other words, \alg collects samples until it has some confidence on the recovered subspace. This idea is considered in more detail in Section \ref{SubsEst}. 
Note that warm-up periods are implicitly assumed in most \SLB algorithms since the given bounds are not meaningful for short periods of time.

\subsection{Subspace Estimation} \label{SubsEst}
At each round, the algorithm predicts the $m$-dimensional subspace that the true action vectors belong, using the action vectors collected up to that round. In particular, at round $t$, the algorithm uses PCA over $tK$ action vectors observed so far, $\uplus_{i=1}^t D_i$. It calculates $\hat{V}_t$ which is the matrix of top $m$ eigenvectors of $\frac{1}{tK} \sum_{\hat{x} \in \uplus_{i=1}^t D_i} \hat{x}\hat{x}^T$, thus $\spn(\hat{V}_t)$ is the predicted $m$-dimensional subspace. Then, $\hat{V}_t$ is used to create the estimated projection matrix associated to this subspace, $\hat{P}_t \coloneqq \hat{V}_t \hat{V}_t^T$. 

As the agent observes more action vectors, the estimated projection matrix becomes more accurate. The accuracy of $\hat{P}_t$ is measured by the projection error $\|\hat{P}_t - P \|_2$. As more action vectors are collected, $\|\hat{P}_t - P \|_2$ shrinks. Since $P$ is not known, $\|\hat{P}_t - P \|_2$ cannot be calculated directly. Thus, \alg calculates a high-probability upper bound on the projection error. Using the derived bound, \alg deploys confidence in the subspace estimation and construct the set of plausible projection matrices $\mathcal{C}_{p,t}$ where $\hat{P}_t$ and $P$ both lie in with high probability. The construction of $\mathcal{C}_{p,t}$ is reliant on the structural properties of the problem and the number of samples $tK$. We analyze these properties in Section \ref{ProjErrAnaly}.

\subsection{Confidence Set Construction} \label{ConfSetCons}
At each round, \alg creates two confidence sets for the model parameter $\theta_*$. First, it tries to exploit a possible $m$-dimensional hidden subspace structure. Thus, it searches for a high probability confidence set, $\mathcal{C}_{m,t}$, that lies around the estimated subspace at round $t$. Using the history of action-reward pairs, the algorithm solves a regularized least squares problem in the estimated subspace and obtains $\theta_t$, the estimated parameter vector in $\spn(\hat{V}_t)$. Then it creates the confidence set $\mathcal{C}_{m,t}$ around $\theta_t$, such that $\theta_* \in \mathcal{C}_{m,t}$ with high probability. 

Second, \alg searches for a high probability confidence set in the ambient space without having subspace recovery. It deploys the confidence set generation subroutine of \oful by \citet{abbasi2011improved}. Using the history of action-reward pairs, the algorithm solves another regularized least squares problem but this time in the ambient space and obtains $\hat{\theta}_t$. \alg then creates the confidence set $\mathcal{C}_{d,t}$ centered around $\hat{\theta}_t$ such that $\theta_* \in \mathcal{C}_{d,t}$ with high probability. Finally, \alg takes the intersection of constructed confidence sets to create the main confidence set, $\mathcal{C}_t = \mathcal{C}_{m,t} \cap \mathcal{C}_{d,t}$. $\mathcal{C}_t$ still contains $\theta_*$ with high probability. With this operation, \alg provides a new perspective that if there exists an easily recoverable $m$-dimensional subspace, it exploits that structure to get lower regret than \oful can solely achieve. If it fails to detect such structure or the confidence set is looser than what \oful provides, then it still provides the same regret as \oful.

\subsection{Optimistic Action} \label{OptAct}
For the final step in round $t$, the algorithm chooses an optimistic triplet $(\tilde{P}_t, \hat{X}_t,  \tilde{\theta}_t)$ from the confidence sets created and the current decision set which jointly maximizes the reward: 
\begin{equation} \label{optimistic}
(\tilde{P}_t, \hat{X}_t,  \tilde{\theta}_t) = \argmax_{(P^\prime, x,\theta)\in \mathcal{C}_{p,t} \times D_t \times \mathcal{C}_{t}} (P^\prime x)^T \theta 
\end{equation}
\vspace{-.5\baselineskip}

\section{Theoretical Analysis of \alg} \label{Analysis}

In this section we first state the upper regret bound of \alg which is the main result of the paper. Then we analyze the components that build up to the result. In order to get a meaningful bound, we assume that the expected rewards are bounded. Recalling the quantities defined in (\ref{quantities}), define $\Upsilon$ such that 
\begin{equation} \label{upsilon_def}
\Upsilon = \OO\left( \frac{\alpha \Gamma^2 \sqrt{m} }{K(\lambda_- + \sigma^2)} \right).
\end{equation}
It represents the difficulty of subspace recovery in terms structural properties of \SLB setting, and it is analyzed in Section \ref{RegAnaly}. Using $\Upsilon$, the theorem below states the regret upper bound of \alg. 

\begin{theorem}[Regret Upper Bound of \alg] \label{RegretAnalysis}
Fix any $\delta \in (0,1/6)$. Assume that Assumptions 1 and 2 hold. Also assume that for all $\hat{x}_{t,i} \in D_t$, $\hat{x}_{t,i}^T \theta_* \in [-1,1]$. Then, $\forall t \geq 1$ with probability at least $1-6\delta$, the regret of \alg satisfies 
\begin{equation}
R_t \leq \min \Big \lbrace \wt\OO\left(\Upsilon \sqrt{t}\right) , \wt\OO\left(d \sqrt{t}\right)\Big\rbrace.
\end{equation}
\end{theorem}
The proof of the theorem involves two main pieces: the projection error analysis and the construction of projected confidence sets. They are analyzed in Sections \ref{ProjErrAnaly} and \ref{ConfSetCre} respectively. Finally, in Section \ref{RegAnaly} their role in the proof of Theorem \ref{RegretAnalysis} is explained and the meaning of the result is discussed.

\subsection{Projection Error Analysis} \label{ProjErrAnaly}

Consider the matrix $\hat{V}_t^T V$ where $i$th singular value is denoted by $\sigma_{i}(\hat{V}_t^T V)$, such that $\sigma_1(\hat{V}_t^T V) \geq \ldots \geq \sigma_m(\hat{V}_t^T V)$. Extending the definition of inner products of two vectors to subspaces and using Courant-Fischer-Weyl minimax principle, one can define $i$\textit{th principal angle} $\Theta_i$ between $\spn(V)$ and $\spn(\hat{V}_t)$ via 
\begin{equation*}
    \cos\Theta_i(\spn(V), \spn(\hat{V}_t)) = \sigma_i(\hat{V}_t^TV).
\end{equation*}
Using the analysis in \citet{akhiezer2013theory} it can be seen that:
\begin{align}
\|\hat{P}_t - P \|_2 &= \sqrt{\lambda_{max}\bigg(I_m - (\hat{V}_t^TV)^T(\hat{V}_t^TV) \bigg)} \nonumber \\
&= \sqrt{1-\sigma^2_m(\hat{V}_t^TV)} = \sin\Theta_m \label{error=sine}
\end{align}
where $\Theta_m$ is the largest principal angle between the column spans of $V$ and $\hat{V}_t$. Thus, bounding the projection error between two projection matrices is equivalent to bounding the sine of the largest principal angle between the subspaces that they project. In light of this relation, one can use the Davis-Kahan $\sin \Theta$ theorem \citep{davis1970rotation} to bound the projection error. The exact theorem statement can be found in Section \ref{ProjSupplement} in the  Supplementary Material. Informally, the theorem considers a symmetric matrix and its' perturbed version and bounds the sine of the largest principal angle caused by this perturbation. Using Davis-Kahan $\sin \Theta$ theorem, following lemma bounds the finite sample projection error. 
\begin{lemma}[Finite Sample Projection Error] \label{errornorm}
Fix any $\delta \in (0,1/3)$. Let $t_{w,\delta} = \frac{n_{\delta}}{K}$. Suppose Assumption 1 holds. Then with probability at least $1-3\delta$, $\forall t\geq t_{w,\delta}$,
\begin{equation} \label{projmain}
\|\hat{P}_t - P \|_2 \leq \frac{\phi_{\delta}}{\sqrt{t}} \quad , \text{ where } \phi_{\delta} = 2\Gamma \sqrt{\frac{\alpha}{K} \log \frac{2d}{\delta} }. 
\end{equation} 
\end{lemma}
The lemma and it's proof are along the same lines of Corollary 2.9 of \citet{vaswani2017finite}. However, we improve the bound on the projection error by using the Matrix Chernoff Inequality \citep{tropp2015introduction} and provide the precise problem dependent quantities in the bound which are required for defining the minimum number of samples for the warm-up period and the construction of confidence sets for $\theta_*$. Note that as discussed in Section \ref{SubsEst}, (\ref{projmain}) defines the confidence set $\mathcal{C}_{p,t}$ for all $t \geq t_{w,\delta}$. The general version of the lemma and the details of the proof are given in Section \ref{ProjSupplement} of the Supplementary Material, but here we provide a proof sketch. 

Up to round $t$, the agent observes $tK$ action vectors in total within the decision sets. Using PCA, \alg estimates an $m$-dimensional subspace spanned by top $m$ eigenvectors of the sample covariance matrix of $tK$ action vectors and obtain the projection matrix $\hat{P}_t$ for that subspace. In order to derive Lemma \ref{errornorm}, we first carefully pick two symmetric matrices such that the span of their first $m$ eigenvectors are equivalent to subspaces that $P$ and $\hat{P}_t$ project to. Using Davis-Kahan $\sin\Theta$ theorem with matrix concentration inequalities provided by \citet{tropp2015introduction}, we derive the finite sample projection error bound.

Lemma \ref{errornorm} is key to defining the warm-up period duration. Due to equivalence in (\ref{error=sine}), $\|\hat{P}_t - P \|_2 \leq 1$, $\forall t\geq 1$. Therefore, any projection error bound greater than 1 is vacuous. We pick $t_{w,\delta}$ such that with high probability, we obtain theoretically non-trivial bound on projection error. With the given choice of $t_{w,\delta}$, the bound on the projection error in (\ref{projmain}) becomes less than 1 when $t\geq t_{w,\delta}$. After $t_{w,\delta}$, \alg starts to produce non-trivial confidence sets $\mathcal{C}_{p,t}$ around $\hat{P}_t$. However, note that $t_{w,\delta}$ can be significantly big for problems that have structure that is hard to recover, e.g. having $\alpha$ linear in $d$.  

Lemma \ref{errornorm} also brings several important intuitions about the subspace estimation problem in terms of the problem structure. Recalling the definition of $\Gamma$ in (\ref{quantities}), as $g_{\psi}$ decreases, the projection error shrinks since the underlying subspace becomes more distinguishable. Conversely, as $g_x$ diverges from 1, it becomes harder to recover the underlying $m$-dimensional subspace. Additionally, since $\alpha$ is the maximum of the effective dimensions of the true action vector and the perturbation vector, having large $\alpha$ makes the subspace recovery harder and the projection error bound looser, whereas observing more action vectors, $K$ in each round produces tighter bound on $\|\hat{P}_t - P \|_2$. The effects of these structural properties on the subspace estimation translate to confidence set construction and ultimately to regret upper bound. 

\subsection{Projected Confidence Sets} \label{ConfSetCre}
In this section, we analyze the construction of $\mathcal{C}_{m,t}$ and $\mathcal{C}_{d,t}$. For any round $t\geq 1$, define $\hat{\Sigma}_t \coloneqq \sum_{i=1}^{t} \hat{X}_{i} \hat{X}_{i}^T = \mathbf{\hat{X}}_{t} \mathbf{\hat{X}}_{t}^T$. At round $t$, let $A_t \coloneqq \hat{P}_t (\hat{\Sigma}_{t-1} + \lambda I_d) \hat{P}_t $ for $\lambda > 0$. The rewards obtained up to round $t$ is denoted as $\mathbf{r}_{t-1}$. At round $t$, after estimating the projection matrix $\hat{P}_t$ associated with the underlying subspace, \alg tries to find $\theta_t$, an estimate of $\theta_*$, while believing that $\theta_*$ lives within the estimated subspace. Therefore, $\theta_t$ is the solution to the following Tikhonov-regularized least squares problem with regularization parameters $\lambda > 0$ and $\hat{P}_t$,
\begin{equation*}
\theta_t = \argmin_{\theta} \| (\hat{P}_t \mathbf{\hat{X}}_{t-1})^T \theta - \mathbf{r}_{t-1} \|_2^2 + \lambda \|\hat{P}_t \theta \|_2^2.
\end{equation*}
Notice that regularization is applied along the estimated subspace. Solving for $\theta$ gives $\theta_t = A_t^{\dagger} \big(\hat{P}_t \mathbf{\hat{X}}_{t-1} \mathbf{r}_{t-1} \big)$. Define $L$ such that for all $t\geq 1$ and $i \in [K]$, $\| \hat{x}_{t,i} \|_2 \leq L$ and let $\gamma = \frac{L^2}{\lambda \log \big(1+ \frac{L^2}{\lambda}\big)}$. The following theorem gives the construction of projected confidence set, $\mathcal{C}_{m,t}$, which is an ellipsoid centered around $\theta_t$ which contains $\theta_*$ with high probability.  
\begin{theorem}[Projected Confidence Set Construction] \label{main}
Fix any $\delta \in (0,1/4)$. Suppose Assumptions 1 \& 2 hold, and $\forall t\geq 1$ and $i \in [K]$, $\| \hat{x}_{t,i} \|_2 \leq L$. If $\| \theta_* \|_2 \leq S $ then, with probability at least $1-4\delta$, $\forall t \geq t_{w,\delta}$, $\theta_*$ lies in the set 
\begin{align}
\mathcal{C}_{m,t}& = \bigg \{ \theta \in \mathbb{R}^d : \| \theta_t - \theta \|_{A_t} \leq \beta_{t,\delta} \bigg \}, \text{ where} \nonumber \\
\beta_{t,\delta} &= R\sqrt{ 2 \log \bigg( \frac{1}{\delta} \bigg ) + m \log \bigg ( 1 + \frac{t L^2}{m \lambda} \bigg )  }  \nonumber  \\
&+  L S \phi_{\delta} \sqrt{ \gamma m \log \bigg ( 1 + \frac{t L^2}{m \lambda} \bigg )} + S \sqrt{\lambda}.
\label{ellips2} 
\end{align}
\end{theorem}
The detailed proof and a general version of the theorem are given in Section \ref{ConfConSupp} of the Supplementary Material. We will highlight the key aspects in here. The overall proof follows a similar machinery used by \citet{abbasi2011improved}. Specifically, the first term of $\beta_{t,\delta}$ in (\ref{ellips2}) is derived similarly by using the self-normalized tail inequality. However, since at each round \alg projects the past actions to an estimated $m$-dimensional subspace to estimate $\theta_*$, $d$ is replaced by $m$ in the bound. While enjoying the benefit of projection, this construction of the confidence set suffers from the finite sample projection error, \textit{i.e.}, uncertainty in the subspace estimation. This effect is observed via second term in (\ref{ellips2}). The second term involves the confidence bound for the estimated projection matrix, $\phi_{\delta}$.  This is critical in determining the tightness of the confidence set on $\theta_*$. As discussed in Section \ref{ProjErrAnaly}, $\phi_{\delta}$ reflects the difficulty of subspace recovery of the given problem and it depends on the underlying structure of the problem and \SLB. This shows that as estimating the underlying subspace gets more difficult, having a projection based approach in the construction of the confidence sets on $\theta_*$ provides looser bounds. 

In order to tolerate the possible difficulty of subspace recovery, \alg also constructs $\mathcal{C}_{d,t}$, which is the confidence set for $\theta_*$ without having subspace recovery. The construction of $\mathcal{C}_{d,t}$ follows \oful by \citet{abbasi2011improved}. Let $Z_t = \hat{\Sigma}_{t-1} + \lambda I_d$. The algorithm tries to find $\hat{\theta}_t$ which is the $\ell^2$-regularized least squares estimate of $\theta_*$ in the ambient space. Thus, $\hat{\theta}_t = Z_t^{-1} \mathbf{\hat{X}}_{t-1} \mathbf{r}_{t-1}$. Construction of $\mathcal{C}_{d,t}$ is done under the same assumptions of Theorem \ref{main}, such that with probability at least $1-\delta$, $\theta_*$ lies in the set
\begin{align*}
\mathcal{C}_{d,t} &= \bigg \{ \theta \in \mathbb{R}^d : \| \hat{\theta}_t - \theta \|_{Z_t} \leq \Omega_{t,\delta} \bigg \} \quad \text{where} \\
\Omega_{t,\delta} &= R\sqrt{ 2 \log \bigg( \frac{1}{\delta} \bigg ) + d \log \bigg ( 1 + \frac{t L^2}{m \lambda} \bigg )  } + S \sqrt{\lambda}.
\end{align*}
The search for an optimistic parameter vector happens in the intersection of $\mathcal{C}_{m,t}$ and $\mathcal{C}_{d,t}$. Notice that $\theta_* \in \mathcal{C}_{m,t} \cap \mathcal{C}_{d,t}$ with probability at least $1-5\delta$. Optimistically choosing the triplet, $(\tilde{P}_t, \hat{X}_t,  \tilde{\theta}_t)$, within the described confidence sets gives \alg a way to tolerate the possibility of failure in recovering an underlying structure. If confidence set $\mathcal{C}_{m,t}$ is loose or \alg is not able to recover an underlying structure, then $\mathcal{C}_{d,t}$ provides the useful confidence set to obtain desirable learning behavior.

\subsection{Regret Analysis} \label{RegAnaly}
Now that the confidence set constructions and the decision making procedures of \alg are explained, it only remains to analyze the regret of \alg. Using the intersection of $\mathcal{C}_{m,t}$ and $\mathcal{C}_{d,t}$ as the confidence set at round $t$, gives \alg the ability to obtain the lowest possible instantaneous regret among both confidence sets. Therefore, the regret of \alg is upper bounded by the minimum of the regret upper bounds on the individual strategies. Using only $\mathcal{C}_{d,t}$ is equivalent to following \oful and the regret analysis can be found in \citet{abbasi2011improved}. The regret analysis of using only the projected confidence set $\mathcal{C}_{m,t}$ is the main contribution of this work. It follows the standard regret decomposition into instantaneous regret components. However, due to having different estimated projection matrices in each round, the derivation of the bound uses a different strategy involving the Matrix Chernoff Inequality~\citep{tropp2015introduction}. The detailed analysis of the regret upper bound and the proof can be found in Section \ref{RegretSupple} of the Supplementary Material. Here we elaborate more on the nature of the regret obtained by using projected confidence sets only, \textit{i.e.} first term in Theorem \ref{RegretAnalysis}, and discuss the effect of $\Upsilon$ in particular. 

$\Upsilon$ is the reflection of the finite sample projection error at the beginning of the algorithm. It captures the difficulty of subspace recovery based on the structural properties of the problem and determines the regret of deploying projection based methods in \SLB{}s. Recall that $\alpha$ is the maximum of the effective dimensions of the true action vector and the perturbation vector. Depending on the structure of the problem, $\alpha$ can be $\OO(d)$, e.g., the perturbation can be effective in many dimensions, which prevents the projection error from shrinking; thus, causes $\Upsilon = \OO(d\sqrt{m})$ resulting in $\wt\OO(d\sqrt{m t})$ regret. The eigengap within the true action vectors $g_x$ and the eigengap between the true action vectors and the perturbation vectors $g_{\psi}$ are critical factors that determine the identifiability of the hidden subspace. As $\sigma^2$ increases, the subspace recovery becomes harder since the effect of perturbation increases. Conversely, as $\lambda_-$ increases, the underlying subspace becomes easier to identify. These effects are significant on the regret of \alg and they are captured by $\Gamma^2$ in $\Upsilon$.
Moreover, having finite samples to estimate the subspace affects the regret bound through $\Upsilon$. Due to the nature of \SLB, this is unavoidable and it scales the final regret by $1/K$. Overall, with all these elements, $\Upsilon$ represents the hardness of using PCA based methods in dimensionality reduction in \SLB{}s. 

Theorem \ref{RegretAnalysis} states that if the underlying structure is easily recoverable, e.g. $\Upsilon = \OO(m)$, then using PCA based dimension reduction and construction of confidence sets provide substantially better regret upper bound for large $d$. If that is not the case, then due to the best of the both worlds approach provided by \alg, the agent obtains the best possible regret upper bound. Note that the bound for using only $\mathcal{C}_{m,t}$ is a worst case bound and as we present in Section \ref{Experiment}, in practice \alg can give significantly better results.

\section{Experiments} \label{Experiment}

\begin{figure*}[!htb]
\captionsetup{justification=centering }
\centering
\subfloat[][MNIST Regret  \\ Comparison for $d\medop{=}1000$]{\label{fig:regretmnist}\includegraphics[width=57.3mm]{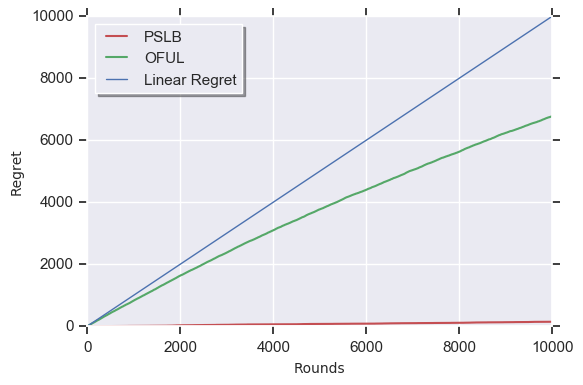}} 
\centering 
\subfloat[][CIFAR-10 Regret \\ Comparison for $d\medop{=}1000$]{\label{fig:regretcifar}\includegraphics[width=57.3mm]{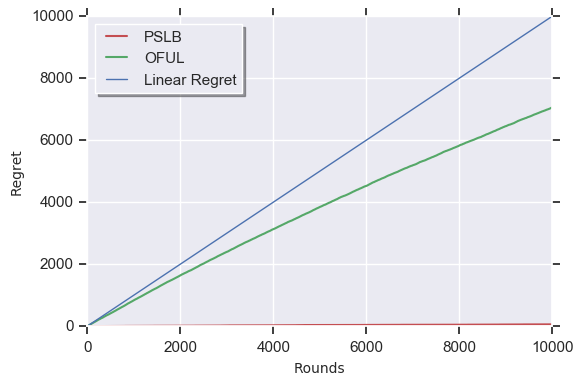}}
\centering
\subfloat[][ImageNet Regret \\ Comparison for $d\medop{=}100$]{\label{fig:regretimagenet}\includegraphics[width=57.2mm]{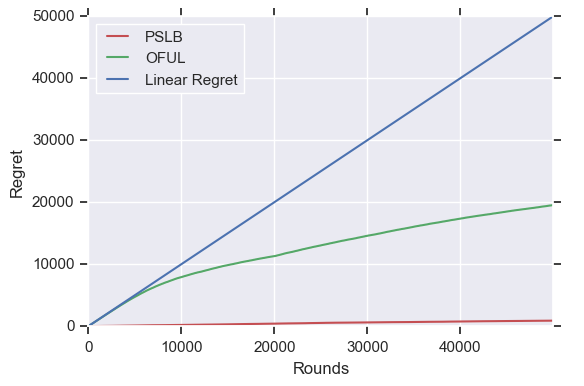}} \vspace{-.8\baselineskip}
\subfloat[][MNIST Model Accuracy \\ Comparison for $d\medop{=}1000$]{\label{fig:classmnist}\includegraphics[width=58mm]{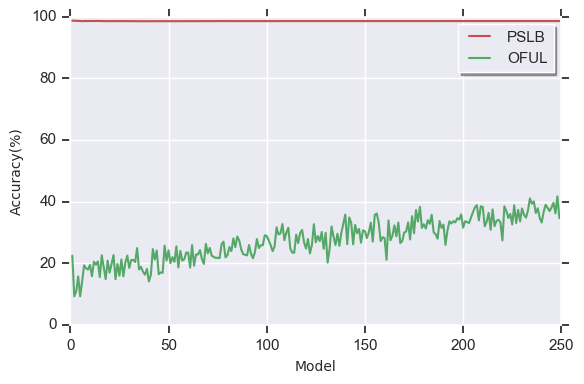}}
\centering
\subfloat[][CIFAR-10 Model Accuracy \\ Comparison for $d\medop{=}1000$]{\label{fig:classcifar}\includegraphics[width=55mm]{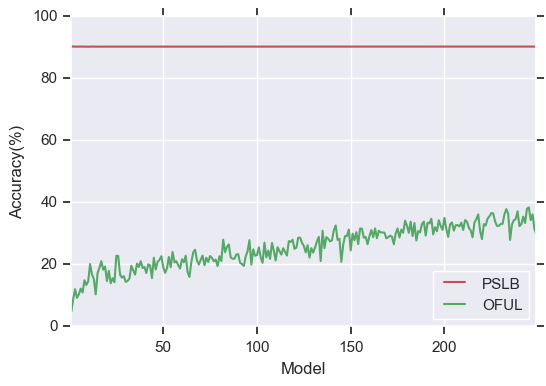}}
\subfloat[][ImageNet Model Accuracy \\ Comparison for $d\medop{=}100$]{\label{fig:classimagenet}\includegraphics[width=55mm]{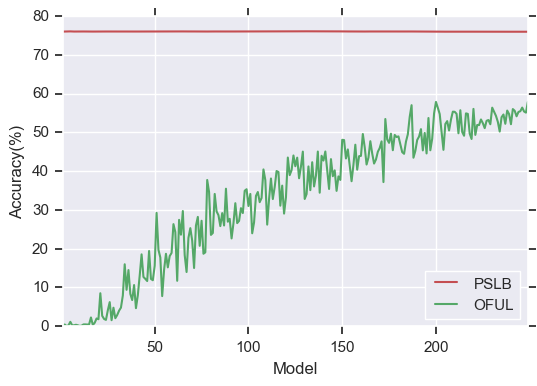}}\vspace{-0.5\baselineskip}
\captionsetup{justification=justified }
\caption{\hfill Regret and Optimistic Model Accuracy Comparisons of \alg and \OFUL on MNIST, CIFAR-10 and ImageNet}\vspace{-0.7\baselineskip}\caption*{\small Top row: Regret of \alg vs. Regret of \OFUL in \SLB setting constructed from image classification tasks. \alg tries to recover $m\medop{=}1$ dimensional subspace which reduces the complexity of \SLB and results in very few committed mistakes. Due to lack of additional knowledge besides rewards obtained from chosen actions, \OFUL starts with linear regret and commits significant amount of mistakes. \\ Bottom row: Image classification accuracy of periodically sampled optimistic models of \alg and \OFUL over all images in datasets. The ability to reduce the complexity of learning task helps \alg to learn the best possible underlying linear model just in few rounds whereas \OFUL requires more action-reward pairs to get an accurate estimate.}
\label{fig:main}
\vspace{-0.7\baselineskip}
\end{figure*} 

In the experiments, we study MNIST, CIFAR-10 and ImageNet datasets and use them to create the decision sets for the \SLB setting. A simple 5-layer CNN, a pre-trained ResNet-18 and a pre-trained ResNet-50 are deployed respectively for MNIST, CIFAR-10 and ImageNet. Before training, we modify the architecture of the representation layer (the layer before the final layer) to make it suitable for the \SLB study and obtain decision sets for each image.

Consider a standard network whose dimension of the representation layer is $d$. Therefore, the final layer for $K$ class classification is fully connected and it is a $d\medop{\times}K$ matrix that outputs $K$ numbers to be used for classification. In this study, instead of having a final layer of $d\medop{\times}K$ matrix, we construct the final layer as a $d$-dimensional vector and make the feature representation layer a $Kd$ dimensional vector. We treat this vector as the concatenation of $K$ $d$-dimensional contexts \textit{i.e.}, $[\hat{x}_1,\ldots,\hat{x}_K]$. The final $d$-dimensional layer is $\theta_*$ of the \SLB, where the logit for each class is computed as an inner product of the class context $\hat{x}_i$ and $\theta_*$. We train these architectures for different $d$ values using cross entropy loss. Here we provide results for MNIST and CIFAR-10 with $d\medop{=}1000$ and ImageNet with $d\medop{=}100$.

Removing the final layer, the resulting trained networks are used to generate the feature representations of each image for each class which produces the decision sets at each time step of \SLB. Since MNIST and CIFAR-10 have 10 classes, in each decision set we obtain 10 action vectors where each of them are segments in the representation layer. On the other hand, from the ImageNet dataset we get 1000 actions per decision set due to 1000 classes in the datasets. In the \SLB setting, the agent receives a reward of 1 if it chooses the right action, which is the segment in the representation layer corresponding to correct label according to trained network, and 0 otherwise. We apply both \alg and \OFUL on these \SLB{}s. We measure the regret by counting the number of mistakes each algorithm makes. To come up with the optimistic choice of action at each time step, both of these algorithms requires solving an inner optimization problem. To mitigate the burden of these computation costs, we sample many models from the confidence sets and choose the most optimistic model among the sampled ones.

Through computing PCA of the empirical covariance matrix of the action vectors, surprisingly we found that projecting action vectors onto the $1$-dimensional subspace defined by the dominant eigenvector is sufficient for these datasets in the \SLB setting; thus, $m = 1$. During the experiments \alg tried to recover a $1$-dimensional subspace using the action vectors collected. We present the regrets obtained by \alg and \OFUL for MNIST, CIFAR-10 and ImageNet in Figure~\ref{fig:regretmnist},~\ref{fig:regretcifar},~\ref{fig:regretimagenet} respectively. With the help of subspace recovery and projection, \alg provides a massive reduction in the  dimensionality of the \SLB problem and immediately estimates a fairly accurate model for $\theta_*$. On the other hand, \OFUL naively tries to sample from all dimensions in order to learn $\theta_*$. This difference yields orders of magnitude improvement in regret. During the \SLB experiment, we also sample the optimistic models that are chosen by \alg and \OFUL. We use these models to test the model accuracy of the algorithms, \textit{i.e.} perform classification over all images in dataset. The optimistic model accuracy comparisons are depicted in 
Figure~\ref{fig:classmnist},~\ref{fig:classcifar},~\ref{fig:classimagenet}. These portray the learning behavior of \alg and \OFUL. Using projection, \alg learns the underlying linear model in the first few rounds, whereas \OFUL suffers from high-dimension of \SLB framework and lack of knowledge besides chosen action-reward pairs. We extend these experiments for settings with $d=100,500,1000$ and $m=1,2,4,8,16$ which can be found in Section \ref{experiment_more}.

\section{Related Work}
The primary class of partial information problems is the multi-arm bandit (MAB).~\citet{robbins1985some} introduces the standard stochastic MAB and~\citet{lai1985asymptotically} studies the asymptotic property of learning algorithms on this class.  Stochastic  MABs are a special case of \SLB when the arms representations are orthogonal to each other. For finite sample regime,~\citet{auer2002finite} deploys the principle of \OFU and provide finite sample guarantee for MABs. 
\citet{auer2002using} deploys the same principle to provide regret guarantee for MABs with the linear pay-off.  This principle is realized as the primary approach even for more general problems such as Linear Quadratic systems~\citep{abbasi2011regret} and Markov Decision Processes~\citep{jaksch2010near}. 

The study of linear bandit problems extends to various algorithms and environment settings \citep{dani2008stochastic,rusmevichientong2010linearly,li2010contextual}. \citet{kleinberg2010regret} studies the class of problems when the decision set changes time to time, while \citet{dani2008stochastic} studies this problem when the decision set provides a set of fixed actions. Further analysis in the area extend these approaches to classes where there are more structures in the problem setup. In traditional decision-making problems, where hand engineered feature representations are provided, sparsity in the linear function is a valid structure. Sparsity, as the key in high-dimensional conventional structured linear bandits, conveys series of successes in classical settings \citep{abbasi2012online,carpentier2012bandit}. In recommendations systems, where a set of users and items are given, \citet{gopalan2016low} consider the low-rank structure of the user-item preference matrix and provide an algorithm which exploits this further structure. 

To the best of our knowledge, there are no hidden low-dimensional subspace assumptions on actions and/or unknown weight vector in literature for \SLB. On the other hand, subspace recovery and dimension reduction problems are well studied in the literature. Several linear and nonlinear dimension reduction methods have been proposed such as PCA~\citep{pearson1901liii}, independent component analysis~\citep{hyvarinen2000independent}, random projections~\citep{candes2006near} and non-convex robust PCA~\citep{netrapalli2014non}. Among the linear dimension reduction techniques, PCA is the simplest, yet most widely used method. Analysis of PCA based methods mostly focus on the asymptotic results ~\citep{anderson1963asymptotic,jain2016streaming}. However, in the settings like \SLB with finite number of arms, it is necessary to have finite sample guarantees for the application of PCA. In the literature, among few finite sample PCA works, \citet{nadler2008finite} provides finite sample guarantees for one-dimensional PCA, whereas \citet{vaswani2017finite} extends it to larger dimensions with various noise models.

\section{Conclusion}
In this paper, we study a linear subspace structure in the action set of an \SLB problem. We deploy PCA based projection to exploit the immense number of unsupervised actions in the decision sets and learn the underlying subspace. We proposed \alg, a \SLB algorithm which utilizes the subspace estimated through PCA to improve the regret upper bound of \SLB{} problems. If such structure does not exist or is hard to recover, then the \alg reduces to the standard \SLB algorithm, \OFUL. We empirically study MNIST, CIFAR-10 and ImageNet datasets to create \SLB framework from image classification tasks. We test the performance of \alg versus \OFUL in the \SLB setting created. We show that when DNNs produce features of the actions, a significantly low dimensional structure is observed. Due to this structure, we showed that \alg substantially outperforms \OFUL and converges to an accurate model while \OFUL still struggles to sample in high dimensions to learn the underlying parameter vector.

In this work, we studied the class of linear subspace structures. In the future work, we plan to extend this line of study to the general class of low dimensional manifold structured problems. \citet{bora2017compressed} peruse a similar approach for compression problems. While optimism is the primary approach in the theoretical analyses of \SLB{}s, it mainly poses a computationally intractable internal optimization problem. An alternative method is Thompson sampling, a practical algorithm for \SLB{}s. In future work, we plan to deploy Thompson sampling and mitigate the computational complexity of \alg. 

\newpage
\bibliography{main}
\bibliographystyle{icml2019}
\newpage
\onecolumn
\appendix
\section{Projection Error Analysis, Proof of Lemma \ref{errornorm} } \label{ProjSupplement}
In this section, we provide the general version of Lemma \ref{errornorm} with the proof details. As stated in the main text, in order to bound the projection error, we will use Davis-Kahan $\sin \Theta$ theorem which states the following: 
\vspace{3mm} 
\begin{theorem}[\citep{davis1970rotation}] \label{sinTheta}
Let $S, H \in \mathbb{R}^{d \times d}$ be symmetric matrices, such that $\hat{S} = S + H$. The eigenvalues of $S$ and $\hat{S}$ are $\lambda_1 \geq \ldots \geq \lambda_m \geq \ldots \geq \lambda_d$ and $\hat{\lambda}_1 \geq \ldots \geq \hat{\lambda}_m \geq \ldots \geq \hat{\lambda}_d$ respectively. Define the eigendecompositions of $S$ and $\hat{S}$: 
\begin{equation*}
S = [U \quad U_o] \left[
\begin{array}{cc}
\Lambda & 0 \\ 
0 & \Lambda_o
\end{array}\right] [U \quad U_o]^T 
\end{equation*}

\begin{equation*}
\hat{S} = [\hat{U} \quad \hat{U}_o] \left[
\begin{array}{cc}
\hat{\Lambda} & 0 \\ 
0 & \hat{\Lambda}_o
\end{array}\right] [\hat{U} \quad \hat{U}_o]^T
\end{equation*}
where $\Lambda$ and $\hat{\Lambda}$ are diagonal matrices with first $m$ eigenvalues of $S$ and $\hat{S}$ respectively. $U = (u_1, \ldots, u_m) \in \mathbb{R}^{d \times m}$ and $\hat{U} = (\hat{u}_1, \ldots, \hat{u}_m) \in \mathbb{R}^{d \times m}$ denote the corresponding eigenvectors. 
Define 
\begin{equation*}
\delta \coloneqq \inf\{|\hat{\lambda} - \lambda| : \lambda \in [\lambda_m, \lambda_1], \hat{\lambda} \in (-\infty, \hat{\lambda}_{m+1}] \}.
\end{equation*}
If $\delta > 0$, then $\sin \Theta_{m}$, sine of the largest principal angle between the column spans of $U$ and $\hat{U}$, can be upper bounded as
\begin{equation} \label{DKsin}
\sin \Theta_{m} \leq \frac{\| \hat{S}U - U\Lambda \|_2}{\delta} = \frac{\| \hat{S}U - U\Lambda \|_2}{|\lambda_m - \hat{\lambda}_{m+1}|}.
\end{equation}

\end{theorem}
Notice that in order to use Davis-Kahan $\sin \Theta$ theorem in our setting, we need to pick 2 symmetric matrices $S$ and $\hat{S}$ such that their first $m$ eigenvectors has the same span with the subspaces that $P$ and $\hat{P}$ project to. Followed by these choices, in order to get a non-trivial bound we require a significant eigengap between $\lambda_m$ and $\hat{\lambda}_{m+1}$, due to denominator in ($\ref{DKsin}$). We use the following matrix concentration inequalities to maintain an eigengap with high probability. 
\vspace{6mm}
\begin{theorem}[Matrix Chernoff Inequality; \citep{tropp2015introduction}] \label{chernoff}
Consider a finite sequence $\{X_k\}$ of independent, random, symmetric matrices in $\mathbb{R}^{d \times d}$. Assume that $\lambda_{min}(X_k) \geq 0$ and $\lambda_{max}(X_k) \leq L$ for each index k. Introduce the random matrix $Y = \sum_k X_k$. Let $\mu_{min}$ denote the minimum eigenvalue of the expectation $\mathbb{E}[Y]$, 
\begin{equation*}
\mu_{min} = \lambda_{min}\big(\mathbb{E}[Y]\big) = \lambda_{min}\bigg(\sum_k \mathbb{E}[X_k]\bigg).
\end{equation*}
Then,
\begin{equation*}
\Pr\bigg[ \lambda_{min}(Y) \leq \epsilon \mu_{min} \bigg] \leq d \exp\big(-(1-\epsilon)^2\frac{\mu_{min}}{2L}\big) \qquad \text{for } \epsilon \in [0,1).
\end{equation*}

\end{theorem}

\vspace{6mm}
\begin{theorem}[Corollary of Matrix Bernstein; \citep{tropp2015introduction}] \label{bernstein}
Consider a set of $n$ i.i.d. realization of a $d_1 \times d_2$ random matrix $R$, as $\lbrace R_1,\ldots,R_n\rbrace$. If $\mathbb{E}[R]$ is bounded, $\|R\|_2 \leq L$ almost surely, with second moment of 
\begin{equation*}
m_2(R) = \max \bigg \{ \| \mathbb{E}[RR^T] \|_2 ,  \| \mathbb{E}[R^TR] \|_2  \bigg \}.
\end{equation*}
Then, for all $t\geq 0$,
\begin{equation*}
\Pr \bigg[ \| \frac{1}{n} \sum_{i=1}^n R_i - \mathbb{E}[R] \|_2  \geq t \bigg] \leq (d_1 + d_2) \exp \bigg ( \frac{-nt^2/2}{m_2(R) + 2Lt/3}\bigg )
\end{equation*}

\end{theorem}
\newpage
Define $t_{\min,\delta} = \bigg(\sqrt{\frac{2d_x g_x}{K}\log \frac{m}{\delta}} + \Gamma \sqrt{\frac{\alpha}{K}\log \frac{2d}{\delta}} \bigg)^2$. Now that we have the required machinery, we present general version of Lemma \ref{errornorm}.

\begin{lemma} \label{errornormSupple}
Fix any $\delta \in (0,1/3)$. Suppose that Assumption 1 holds. Then with probability at least $1-3\delta$, 
\begin{equation*} 
\|\hat{P}_t - P \|_2 \leq \Phi_{t,\delta}, \qquad \forall t\geq t_{w,\delta},
\end{equation*}
where 
\begin{equation} \label{error_general}
\Phi_{t,\delta} = \frac{ \Gamma \sqrt{\frac{\alpha }{tK} \log \frac{2d}{\delta}}}{ 1 - \sqrt{\frac{2d_x g_x}{tK}\log\frac{m}{\delta}} - \Gamma \sqrt{\frac{\alpha }{tK}\log \frac{2d}{\delta}} }.
\end{equation}
\end{lemma}

\begin{proof}
We set $\hat{S} = \frac{1}{n}\sum_{i=1}^{n} \hat{x}_i\hat{x}_i^T$ and $S = \frac{1}{n}\sum_{i=1}^{n}  x_i x_i^T + VV^T \Sigma_{\psi} VV^T$ where $n=tK$. Let U be the top $m$ eigenvectors of S. Notice that $\spn(U) = \spn(V)$ and  $\hat{V}_t$ is the matrix of top $m$ eigenvectors of $\hat{S}$. Therefore, one can apply Theorem \ref{sinTheta} with given choices of $S$ and $\hat{S}$, to bound $\|\hat{P}_t - P \|_2 $. Since $\| \hat{S}U - U\Lambda \|_2 = \| (\hat{S} - S)V \|_2$,
\begin{equation*}
\|\hat{P}_t - P \|_2 \leq \frac{\| (\hat{S} - S)V \|_2}{\lambda_m(S) - \lambda_{m+1}(\hat{S})} \stackrel{(1)}{\leq} \frac{\|  (\hat{S} - S)V \|_2}{\lambda_m(S) - \| \hat{S} - S \|_2} \stackrel{(2)}{\leq} \frac{\| \mathbb{E}[\hat{S} - S]V \|_2 + \| \hat{S} - S - \mathbb{E}[\hat{S} - S] \|_2}{\lambda_m(S) - \| \mathbb{E}[\hat{S} - S] \|_2 - \| \hat{S} - S - \mathbb{E}[\hat{S} - S] \|_2 }
\end{equation*}
where (1) follows from Weyl's inequality and the fact that $S$ is rank $m$, $\lambda_{m+1}(S) = \ldots = \lambda_{d} = 0$, and (2) is due to triangle inequality. With the given choices of $S$ and $\hat{S}$ and Assumption 1, we have the following:
\begin{align*}
&\lambda_m(S) \geq \lambda_{m}(\frac{1}{n}\sum_{i=1}^{n}  x_i x_i^T) + \lambda_{\min}(V^T \Sigma_{\psi} V) = \lambda_{m}(\frac{1}{n}\sum_{i=1}^{n}  x_i x_i^T) + \lambda_{\min}(\sigma^2 I_m) = \lambda_{m}(\frac{1}{n}\sum_{i=1}^{n}  x_i x_i^T) + \sigma^2\\
&\hat{S} - S = \frac{1}{n}\sum_{i=1}^{n} \psi_i \psi_i^T + \frac{1}{n}\sum_{i=1}^{n} x_i \psi_i^T + \frac{1}{n}\sum_{i=1}^{n} \psi_i x_i^T  - VV^T \Sigma_{\psi} VV^T \\
&\| \mathbb{E}[\hat{S} - S] \|_2 = \|\sigma^2 I_d - \sigma^2 P \|_2 =  \sigma^2 \qquad \mathbb{E}[\hat{S} - S]V = \Sigma_{\psi} V - VV^T \Sigma_{\psi} V = V_{\perp} V_{\perp}^T \Sigma_{\psi} V = 0 \\
&\hat{S} - S - \mathbb{E}[\hat{S} - S] = \frac{1}{n}\sum_{i=1}^{n} \psi_i \psi_i^T - \Sigma_{\psi} + \frac{1}{n}\sum_{i=1}^{n} x_i \psi_i^T + \frac{1}{n}\sum_{i=1}^{n} \psi_i x_i^T.
\end{align*}
Inserting these expressions we get, 
\begin{equation}
\|\hat{P}_t - P \|_2 \leq \frac{ \| \frac{1}{n}\sum_{i=1}^{n} \psi_i \psi_i^T - \Sigma_{\psi} + \frac{1}{n}\sum_{i=1}^{n} x_i \psi_i^T + \frac{1}{n}\sum_{i=1}^{n} \psi_i x_i^T \|_2}{\lambda_{m}(\frac{1}{n}\sum_{i=1}^{n}  x_i x_i^T) - \| \frac{1}{n}\sum_{i=1}^{n} \psi_i \psi_i^T - \Sigma_{\psi} + \frac{1}{n}\sum_{i=1}^{n} x_i \psi_i^T + \frac{1}{n}\sum_{i=1}^{n} \psi_i x_i^T \|_2 }
\end{equation}

We first bound $\lambda_{m}(\frac{1}{n}\sum_{i=1}^{n}  x_i x_i^T)$. From Assumption 1, $\lambda_{max}\big( x_i x_i^T\big) \leq d_x \lambda_+$ for all $i \in [n]$ and from the model properties, $\lambda_{m}\big( \sum_{i=1}^{n} \mathbb{E}[x_i x_i^T]\big) = n\lambda_{-} $. Using Theorem \ref{chernoff}, one can get that 

\begin{equation}\label{num1}
\Pr\bigg[\lambda_{m}\bigg( \frac{1}{n}\sum_{i=1}^{n} x_i x_i^T \bigg) \leq \lambda_{-} \bigg (1 - \sqrt{\frac{2d_x g_x}{n}\log\frac{m}{\delta}}\bigg )\bigg] \leq \delta.
\end{equation}
Now we consider $ \| \frac{1}{n}\sum_{i=1}^{n} \psi_i \psi_i^T - \Sigma_{\psi} + \frac{1}{n}\sum_{i=1}^{n} x_i \psi_i^T + \frac{1}{n}\sum_{i=1}^{n} \psi_i x_i^T \|_2$. From triangle inequality we have,
\begin{equation*}
\bigg\| \frac{1}{n}\sum_{i=1}^{n} \psi_i \psi_i^T - \Sigma_{\psi} + \frac{1}{n}\sum_{i=1}^{n} x_i \psi_i^T + \frac{1}{n}\sum_{i=1}^{n} \psi_i x_i^T \bigg \|_2 \leq \bigg \| \frac{1}{n}\sum_{i=1}^{n} \psi_i \psi_i^T - \Sigma_{\psi} \bigg \|_2 + 2 \bigg \| \frac{1}{n}\sum_{i=1}^{n} x_i \psi_i^T \bigg \|_2
\end{equation*} 
We will consider each term on the right hand side separately. If Assumption 1 holds, then we have:
\begin{align*}
\mathbb{E}[\psi_{i}\psi_{i}^T] &= \Sigma_{\psi} \\
\| \psi_{i}\psi_{i}^T \|_2 &\leq d_{\psi} \sigma^2 \\
\| \mathbb{E}[ \psi_{i}\psi_{i}^T\psi_{i}\psi_{i}^T] \|_2 &\leq d_{\psi} \sigma^2 \| \mathbb{E}[\psi_{i}\psi_{i}^T] \|_2 = d_{\psi} \sigma^4
\end{align*}
Applying Theorem \ref{bernstein}, we get
\begin{equation} \label{denum1}
\Pr \bigg[ \bigg \|  \frac{1}{n}\sum_{i=1}^{n} \psi_i \psi_i^T - \Sigma_{\psi} \bigg \|_2  \geq 2\sigma^2\sqrt{\frac{d_{\psi}}{n}\log \frac{2d}{\delta}} \bigg ] \leq \delta  \quad \text{ for } 2\sqrt{\frac{d_{\psi}}{n}\log \frac{2d}{\delta}} \leq 1.5.
\end{equation}
Under the same assumption for the second term we have:
\begin{align*}
\mathbb{E}[x_{i}\psi_{i}^T] &= 0\\
\| x_{i}\psi_{i}^T \|_2 &= \sqrt{\lambda_{max}(\psi_{i}x_i^T x_{i}\psi_{i}^T)} \leq \sqrt{  d_x \lambda_{+} d_{\psi} \sigma^2 } \\
\| \mathbb{E}[ x_{i}\psi_{i}^T\psi_{i}x_i^T] \|_2 &\leq d_{\psi} \sigma^2 \| \mathbb{E}[x_{i}x_{i}^T] \|_2 = d_{\psi} \lambda_{+} \sigma^2 \\
\| \mathbb{E}[ \psi_{i}x_{i}^T x_{i}\psi_i^T] \|_2 &\leq  d_x \lambda_{+} \| \mathbb{E}[\psi_{i}\psi_{i}^T] \|_2 \leq d_x \lambda_{+} \sigma^2
\end{align*}
Once again applying Theorem \ref{bernstein},
\begin{equation} \label{denum2}
\Pr \bigg[ \bigg \|  \frac{1}{n}\sum_{i=1}^{n} x_{i}\psi_{i}^T \bigg \|_2  \geq 2\sqrt{\lambda_{+} \sigma^2} \sqrt{\frac{\alpha}{n}\log\frac{2d}{\delta}} \bigg ] \leq \delta  \quad \text{ for } 2\sqrt{\frac{\alpha}{n}\log\frac{2d}{\delta}} \leq 1.5.
\end{equation}
Finally, combining (\ref{num1}), (\ref{denum1}), (\ref{denum2}) and using union bound, for any round $t \geq t_{\min, \delta}$, we get: 
\begin{equation*}
\|\hat{P}_t - P \|_2 \leq \min \Bigg ( \frac{ \Gamma \sqrt{\frac{\alpha }{tK} \log \frac{2d}{\delta}}}{ 1 - \sqrt{\frac{2d_x g_x}{tK}\log\frac{m}{\delta}} - \Gamma \sqrt{\frac{\alpha }{tK}\log \frac{2d}{\delta}} } , 1 \Bigg ) \quad \text{w.p. } 1- 3\delta.
\end{equation*}
As explained in the main text, due to equivalence between the projection error and the sine of the largest angle between the subspaces, the projection error is always bounded by 1. Thus, in our bound we impose that constraint. Notice that lower bound on $t$ is to satisfy that concentration inequalities provide meaningful results. In other words, $Kt_{\min,\delta}$ is the number of samples required to have non-negative denominator to use Davis-Kahan $\sin \Theta$ theorem. However, observe that we need $Kt_{w,\delta}$ samples to obtain high probability error bound which is non-trivial, \textit{i.e.} less than 1 and $t_{w,\delta} = 4 t_{\min, \delta}$. Therefore, for any $t \geq t_{w,\delta}$ the stated bound (\ref{error_general}) in the lemma holds with high probability and for any $1\leq t \leq t_{w,\delta}$ we bound the projection error by 1. 

Only step remaining to show that lemma holds $\forall t \geq t_{w,\delta}$. This requires  an argument which shows that this bound is valid uniformly over all rounds. To this end, we use stopping time construction, which goes back at least to \citet{freedman1975tail}. 

Define the bad event, 
\begin{equation*}
E_{\tau}(\delta) = \bigg \{ \|\hat{P}_{\tau} - P \|_2 > \frac{ \Gamma \sqrt{\frac{\alpha }{\tau K} \log \frac{2d}{\delta}}}{ 1 - \sqrt{\frac{2d_x g_x}{\tau K}\log\frac{m}{\delta}} - \Gamma \sqrt{\frac{\alpha }{\tau K}\log \frac{2d}{\delta}} } \bigg \}.
\end{equation*}
We are interested in the probability of $\bigcup\limits_{t \geq t_{w,\delta}} E_{t}(\delta)$. Define $\tau(\omega) = \min\{t \geq t_{w,\delta} : \omega \in E_{t}(\delta)\}$, with the convention that $\min \emptyset = \infty$. Then, $\tau$ is a stopping time. Thus, $\bigcup\limits_{t \geq t_{w,\delta}} E_{t}(\delta) = \{\omega : \tau(\omega) < \infty \}$. The Lemma \ref{errornormSupple} can be obtained as follows:

\begin{align*}
  \Pr \bigg[ \bigcup\limits_{t \geq t_{w,\delta}} E_{t}(\delta) \bigg] = \Pr[\tau < \infty] &=  \Pr \bigg [ \|\hat{P}_{\tau} - P \|_2 > \frac{ \Gamma \sqrt{\frac{\alpha }{\tau K} \log \frac{2d}{\delta}}}{ 1 - \sqrt{\frac{2d_x g_x}{\tau K}\log\frac{m}{\delta}} - \Gamma \sqrt{\frac{\alpha }{\tau K}\log \frac{2d}{\delta}} }, \tau < \infty \bigg ] \\
  &= \Pr \bigg [ \|\hat{P}_{\tau} - P \|_2 > \frac{ \Gamma \sqrt{\frac{\alpha }{\tau K} \log \frac{2d}{\delta}}}{ 1 - \sqrt{\frac{2d_x g_x}{\tau K}\log\frac{m}{\delta}} - \Gamma \sqrt{\frac{\alpha }{\tau K}\log \frac{2d}{\delta}} } \bigg ] \leq 3\delta.
\end{align*}

Finally, notice that Lemma \ref{errornorm} presented in the main text is direct consequence of having denominator at (\ref{error_general}) greater than $\frac{1}{2}$ for all $t\geq t_{w,\delta}$. 

\end{proof}

\section{Confidence Set Construction Analysis, Proof of Theorem \ref{main}} \label{ConfConSupp}

In this section, we state the general version of Theorem \ref{main} and provide the proof details. First, recall that $A_t = \hat{P}_t (\hat{\Sigma}_{t-1} + \lambda I_d) \hat{P}_t $. Let $B_t$ be a symmetric matrix such that $A_t = \hat{V}_t B_t \hat{V}_t^T$. Notice that $B_t$ is a full rank $m \times m$ matrix. Also define $\bar{A_t} = A_t - \lambda \hat{P}_t = \hat{P}_t \hat{\Sigma}_{t-1}  \hat{P}_t = \hat{V}_t \hat{V}_t^T \hat{\Sigma}_{t-1} \hat{V}_t \hat{V}_t^T = \hat{V}_t \bar{B}_t \hat{V}_t^T$ where $\bar{B}_t = \hat{V}_t^T \hat{\Sigma}_{t-1} \hat{V}_t = B_t - \lambda I_m$. Using these definitions we can now state the general version of Theorem \ref{main} in which also provides the worst case bound presented in the main text as (\ref{ellips2}).

\begin{theorem} \label{mainSupple}
Fix any $\delta \in (0,1/4)$. Suppose Assumption 1 \& 2 hold. If $\| \theta_* \|_2 \leq S $ then, with probability at least $1-4\delta$, $\forall t \geq 1$, $\theta_*$ lies in the set 
\begin{equation*}
\mathcal{C}_{m,t} = \bigg \{ \theta \in \mathbb{R}^d : \| \theta_t - \theta \|_{A_t} \leq \beta_{t,\delta} \bigg \}, 
\end{equation*}
where
\begin{equation}
\beta_{t,\delta} = R\sqrt{ 2 \log \frac{\det(B_{t})^{1/2} \det(\lambda I_m)^{-1/2}  }{\delta}} \\ 
+  S \Phi_{t,\delta} \| (A_t^{\dagger})^{1/2} \hat{P}_t \hat{\Sigma}_{t-1} \|_2  + S\sqrt{\lambda}.  \label{ellips1Supple}
\end{equation}
If $\| \hat{x}_{t,i} \|_2 \leq L$ for all $t\geq 1$ and $i \in [K]$, then with probability at least $1-4\delta$, $\forall t \geq t_{w,\delta}$, $\theta_*$ lies in the same set with 
\begin{equation}
\beta_{t,\delta} = R\sqrt{ 2 \log \bigg( \frac{1}{\delta} \bigg ) + m \log \bigg ( 1 + \frac{t L^2}{m \lambda} \bigg )  } \\
+ 2\Gamma S L \sqrt{\frac{\alpha}{K} \log \frac{2d}{\delta} } \sqrt{ \gamma m \log \bigg ( 1 + \frac{t L^2}{m \lambda} \bigg )}+ S \sqrt{\lambda}. 
\label{ellips2Supple}
\end{equation}
\end{theorem}

\begin{proof}

Let $S_t \coloneqq \sum_{i=1}^{t} \hat{P}_t \hat{X}_{i-1} \eta_{i-1} = \hat{P}_t \mathbf{X}_{t-1} \pmb{\eta}_{t-1}$. From the definition of $\theta_t$ and $r_t$, we get the following:
\begin{align*}
 \theta_t &= A_t^{\dagger} S_t + A_t^{\dagger} \hat{P}_t \hat{\Sigma}_{t-1} P \theta_* \quad \text{ since } \theta_* \in \spn(V) \\
&= A_t^{\dagger} S_t +  A_t^{\dagger} \big ( \hat{P}_t \hat{\Sigma}_{t-1} (\hat{P}_t + P - \hat{P}_t) + \lambda \hat{P}_t - \lambda \hat{P}_t \big)\theta_* \\
 &= A_t^{\dagger} S_t + \hat{P}_t \theta_* +  A_t^{\dagger} (\hat{P}_t \hat{\Sigma}_{t-1} (P - \hat{P}_t)) \theta_* - \lambda A_t^{\dagger} \theta_*.
 \end{align*}
 Using this, we derive the following for $x = A_t(\theta_t - \theta_*)$:
 \begin{align*}
 x^T \theta_t - x^T \theta_* &= x^T A_t^{\dagger} S_t +  x^T A_t^{\dagger} (\hat{P}_t \hat{\Sigma}_{t-1} (P - \hat{P}_t)) \theta_* - \lambda x^T A_t^{\dagger} \theta_* \\
 &= \langle x, S_t \rangle_{A_t^{\dagger}} + \langle x, \hat{P}_t \hat{\Sigma}_{t-1} (P - \hat{P}_t) \theta_* \rangle_{A_t^{\dagger}} - \lambda \langle x,\theta_* \rangle_{A_t^{\dagger}}. 
 \end{align*}
 Using Cauchy-Schwarz inequality, we can upper bound the magnitude of the difference as follows: 
 \begin{align}
 | x^T \theta_t - x^T \theta_* | &\leq \| x \|_{A_t^{\dagger}} \big ( \| S_t \|_{A_t^{\dagger}} + \| \hat{P}_t \hat{\Sigma}_{t-1} (P - \hat{P}_t) \theta_* \|_{A_t^{\dagger}} + \lambda \| \theta_* \|_{A_t^{\dagger}}	\big ) \nonumber \\
 &\leq \| x \|_{A_t^{\dagger}} \big ( \| S_t \|_{A_t^{\dagger}} + \|  (A_t^{\dagger})^{1/2} \hat{P}_t \hat{\Sigma}_{t-1} (P - \hat{P}_t) \theta_* \|_2  + \sqrt{\lambda} \| \theta_* \|_2	\big ) \label{2normline} \\ 
 &\leq \| x \|_{A_t^{\dagger}} \big ( \| S_t \|_{A_t^{\dagger}} + \| (A_t^{\dagger})^{1/2} \hat{P}_t \hat{\Sigma}_{t-1} \|_2 \|P - \hat{P}_t \|_2 \| \theta_* \|_2 + \sqrt{\lambda} \| \theta_* \|_2	\big ) \nonumber \quad \text{Using C.S. again.} 
 \end{align}
Plugging in $x = A_t(\theta_t - \theta_*)$, we get
\begin{equation*}
\| \theta_t - \theta_* \|^2_{A_t} \leq \| A_t(\theta_t - \theta_*) \|_{A_t^{\dagger}} \bigg ( \| S_t \|_{A_t^{\dagger}}	+  \| (A_t^{\dagger})^{1/2} \hat{P}_t \hat{\Sigma}_{t-1} \|_2 \| (P - \hat{P}_t) \|_2 \| \theta_* \|_2 + \sqrt{\lambda} \| \theta_* \|_2	\bigg ).  
\end{equation*}
Since $\| A_t(\theta_t - \theta_*) \|_{A_t^{\dagger}} = \| \theta_t - \theta_* \|_{A_t}$, dividing both sides with $\| \theta_t - \theta_* \|_{A_t}$ gives and using the fact that $\| \theta_* \| \leq S$,
\begin{equation}
\| \theta_t - \theta_* \|_{A_t} \leq  \| S_t \|_{A_t^{\dagger}}	+  S \| (A_t^{\dagger})^{1/2} \hat{P}_t \hat{\Sigma}_{t-1} \|_2 \| (P - \hat{P}_t) \|_2  + S \sqrt{\lambda} \label{conf_terms}
\end{equation}
We will now bound each term in the (\ref{conf_terms}) separately. The first term is projected version of Theorem 1 in \citep{abbasi2011improved} and second term is the additional term appearing in the confidence interval construction due to non-zero projection error. As it can be seen with the knowledge of true projection matrix the confidence interval reduces to the one in \citep{abbasi2011improved} with replacement of $d$ with $m$. We will first provide the theorem that bounds $ \| S_t \|_{A_t^{\dagger}}$ followed by its proof.

\begin{theorem} \label{martin}
For any $\delta > 0$, with probability at least $1-\delta$, for all $t\geq 1$, 
\begin{equation*}
\| S_{t} \|^2_{A_{t}^{\dagger}} \leq 2 R^2 \log \bigg( \frac{\det(B_{t})^{1/2} \det(\lambda I_m)^{-1/2}}{\delta } \bigg ). 
\end{equation*}
\end{theorem}
\begin{proof}
Without loss of generality, assume that $R=1$ since by appropriately scaling $S_t$, this can be achieved. Let $\lambda \in \mathbb{R}^d$ be a Gaussian random vector which is independent of all the other random variables and has covariance matrix $C^{-1} = \frac{1}{\lambda}I_d$. Consider for any $t \geq 0$,
\begin{equation*}
M_t^{\lambda} = \exp \bigg(\lambda^T S_t - \frac{1}{2} \big ( \lambda^T \sum_{i=1}^t \hat{P}_t \hat{X}_{i-1} \big )^2 \bigg) 
\end{equation*}
Define 
\begin{equation*}
M_t = \mathbb{E}_{\lambda}[M_t^{\lambda} | F_{\infty}]
\end{equation*}
where $F_{\infty}$ is the tail $\sigma$-algebra of the filtration, \textit{i.e.} the $\sigma$-algebra generated by the union of the all the events in the filtration. Thus,
\begin{equation*}
M_t = \int_{\mathbb{R}^d}  \exp \bigg(\lambda^T S_t - \frac{1}{2} \lambda^T \hat{P}_t \hat{\Sigma}_{t-1}  \hat{P}_t \lambda \bigg) f(\lambda) d\lambda
\end{equation*}
where $f(\lambda)$ is the pdf of $\lambda$. The following lemma will be crucial in proving the theorem. 
\begin{lemma} \label{expected}
$\mathbb{E} [M_t] \leq 1$ for all $t \geq 1$.
\end{lemma}
\begin{proof}
\begin{align*}
\mathbb{E} [M_t] &= \mathbb{E} \bigg [ \int_{\mathbb{R}^d}  \exp \bigg(\lambda^T S_t - \frac{1}{2} \lambda^T \hat{P}_t \hat{\Sigma}_{t-1}  \hat{P}_t \lambda \bigg) f(\lambda) d\lambda \bigg ] \\
\mathbb{E} [M_t] &= \int_{\mathbb{R}^d} \mathbb{E} \bigg [ \exp \bigg(\lambda^T S_t - \frac{1}{2} \lambda^T \hat{P}_t \hat{\Sigma}_{t-1}  \hat{P}_t \lambda \bigg) \bigg ]  f(\lambda) d\lambda
\end{align*}
If one can show that $\mathbb{E} \bigg [ \exp \bigg(\lambda^T S_t - \frac{1}{2} \lambda^T \hat{P}_t \hat{\Sigma}_{t-1}  \hat{P}_t \lambda \bigg) \bigg ] \leq 1$, then the claim follows. In the following, we use the law of total expectation. 
\begin{align}
\mathbb{E} \bigg [ \exp \bigg(\lambda^T S_t - \frac{1}{2} \lambda^T \hat{P}_t \hat{\Sigma}_{t-1}  \hat{P}_t \lambda \bigg) \bigg ] &= \mathbb{E} \bigg [ \mathbb{E}_{\eta_{t-1}} \bigg [ \exp \bigg(\lambda^T \sum_{i=1}^t \hat{P}_t \hat{X}_{i-1} \eta_{i-1} - \frac{1}{2} \lambda^T \hat{P}_t \big( \sum_{i=1}^t \hat{X}_{i-1} \hat{X}_{i-1}^T \big) \hat{P}_t \lambda \bigg) \bigg | F_{t-1} \bigg ] \bigg] \nonumber \\
&\leq \mathbb{E} \bigg [ \exp \bigg(\lambda^T \sum_{i=1}^{t-1} \hat{P}_t \hat{X}_{i-1} \eta_{i-1} - \frac{1}{2} \lambda^T \hat{P}_t \big( \sum_{i=1}^{t-1} \hat{X}_{i-1} \hat{X}_{i-1}^T \big) \hat{P}_t \lambda \bigg) \bigg ] \label{subgaussianity} \\
&= \mathbb{E} \bigg [ \mathbb{E}_{\eta_{t-2}} \bigg [ \exp \bigg(\lambda^T \sum_{i=1}^{t-1} \hat{P}_t \hat{X}_{i-1} \eta_{i-1} - \frac{1}{2} \lambda^T \hat{P}_t \big( \sum_{i=1}^{t-1} \hat{X}_{i-1} \hat{X}_{i-1}^T \big) \hat{P}_t \lambda \bigg) \bigg | F_{t-2} \bigg ] \bigg] \nonumber \\
&\shortvdotswithin{=} \nonumber \\
&\leq 1. \nonumber
\end{align}
where \ref{subgaussianity} follows from the assumption that $\eta_t$ is conditionally $R$-sub-gaussian.
\end{proof}
We will use Lemma \ref{expected} shortly but we first calculate $M_t$. For a positive definite matrix $K$, define $g(K) \coloneqq \sqrt{(2\pi)^m/\det(K)} = \int_{\mathbb{R}^m} \exp(-\frac{1}{2}x^T K x)dx$. One can calculate $M_t$ as follows,

\small
\begin{align}
M_t &= \int_{\mathbb{R}^d}  \exp \bigg(\lambda^T S_t - \frac{1}{2} \lambda^T \bar{A_t} \lambda \bigg) f(\lambda) d\lambda \nonumber \\
&= \int_{\mathbb{R}^m}  \exp \bigg( \bar{\lambda}^T \hat{V}_t^T \mathbf{X}_t \pmb{\eta}_t - \frac{1}{2} \bar{\lambda}^T \hat{V}_t^T \mathbf{X}_t \mathbf{X}_t^T \hat{V}_t \bar{\lambda} \bigg) f(\bar{\lambda}) d\bar{\lambda} \qquad \text{change of integration with } \bar{\lambda} = \hat{V}_t^T \lambda \nonumber \\
&= \int_{\mathbb{R}^m} \exp \bigg( -\frac{1}{2} \| \bar{\lambda} - \bar{B}_t^{-1} \hat{V}_t^T \mathbf{X}_t \pmb{\eta}_t \|^2_{\bar{B}_t} + \frac{1}{2} \| \hat{V}_t^T \mathbf{X}_t \pmb{\eta}_t \|^2_{\bar{B}_t^{-1}} \bigg ) f(\bar{\lambda}) d\bar{\lambda} \nonumber  \\
&= \frac{\exp \big( \frac{1}{2} \| \hat{V}_t^T \mathbf{X}_t \pmb{\eta}_t \|^2_{\bar{B}_t^{-1}} \big)}{g(\bar{C})} \int_{\mathbb{R}^m} \exp \bigg( -\frac{1}{2} \big ( \| \bar{\lambda}  - \bar{B}_t^{-1} \hat{V}_t^T \mathbf{X}_t \pmb{\eta}_t \|^2_{\bar{B}_t} + \| \bar{\lambda} \|^2_{\bar{C}} \big) \bigg ) d\bar{\lambda} \qquad \text{where } \bar{C} = \hat{V}_t^T C \hat{V}_t \label{pdfremove}\\
&= \frac{\exp \big( \frac{1}{2} \| \hat{V}_t^T \mathbf{X}_t \pmb{\eta}_t \|^2_{\bar{B}_t^{-1}} \big)}{g(\bar{C})} \int_{\mathbb{R}^m} \exp \bigg( -\frac{1}{2} \big ( \| \bar{\lambda}  - (\bar{C} + \bar{B}_t)^{-1} \hat{V}_t^T \mathbf{X}_t \pmb{\eta}_t \|^2_{\bar{C} + \bar{B}_t} + \| \hat{V}_t^T \mathbf{X}_t \pmb{\eta}_t \|^2_{\bar{B}_t^{-1}} - \| \hat{V}_t^T \mathbf{X}_t \pmb{\eta}_t \|^2_{(\bar{C} + \bar{B}_t)^{-1}} \big) \bigg ) d\bar{\lambda}  \label{simpleeq} \\
&= \frac{\exp \big( \frac{1}{2} \| \hat{V}_t^T \mathbf{X}_t \pmb{\eta}_t \|^2_{(\bar{C} + \bar{B}_t)^{-1}} \big)}{g(\bar{C})} \int_{\mathbb{R}^m} \exp \bigg( -\frac{1}{2} \big ( \| \bar{\lambda}  - (\bar{C} + \bar{B}_t)^{-1} \hat{V}_t^T \mathbf{X}_t \pmb{\eta}_t \|^2_{\bar{C} + \bar{B}_t} \big) \bigg ) d\bar{\lambda} \nonumber \\
&= \frac{\exp \big( \frac{1}{2} \| \hat{V}_t^T \mathbf{X}_t \pmb{\eta}_t \|^2_{(\bar{C} + \bar{B}_t)^{-1}} \big)}{g(\bar{C})} g(\bar{C} + \bar{B}_t)  =  \bigg ( \frac{\det(\bar{C})}{\det(\bar{C} + \bar{B}_t)} \bigg ) ^{1/2} \exp \big( \frac{1}{2} \| S_t \|^2_{(C+\bar{A_t})^{\dagger}} \big),  \nonumber 
\end{align}
\normalsize
where (\ref{pdfremove}) follows from the fact that $f(\bar{\lambda}) = \frac{\exp(-\frac{1}{2} \bar{\lambda}^T  \bar{C} \bar{\lambda} )}{\sqrt{(2\pi)^m \det(\bar{C}^{-1})}}$  and (\ref{simpleeq}) follows since 
\begin{align*}
\| \bar{\lambda}  - \bar{B}_t^{-1} \hat{V}_t^T \mathbf{X}_t \pmb{\eta}_t \|^2_{\bar{B}_t} + \| \bar{\lambda} \|^2_{\bar{C}}  &=  \| \bar{\lambda}  - (\bar{C} + \bar{B}_t)^{-1} \hat{V}_t^T \mathbf{X}_t \pmb{\eta}_t \|^2_{\bar{C} + \bar{B}_t} + \| \bar{B}_t^{-1} \hat{V}_t^T \mathbf{X}_t \pmb{\eta}_t \|^2_{\bar{B}_t} - \| \hat{V}_t^T \mathbf{X}_t \pmb{\eta}_t \|^2_{(\bar{C} + \bar{B}_t)^{-1}} \\ &= \| \bar{\lambda}  - (\bar{C} + \bar{B}_t)^{-1} \hat{V}_t^T \mathbf{X}_t \pmb{\eta}_t \|^2_{\bar{C} + \bar{B}_t} + \| \hat{V}_t^T \mathbf{X}_t \pmb{\eta}_t \|^2_{\bar{B}_t^{-1}} - \| \hat{V}_t^T \mathbf{X}_t \pmb{\eta}_t \|^2_{(\bar{C} + \bar{B}_t)^{-1}}.
\end{align*}
Consider the following equivalence: 
\begin{align}
\Pr \bigg [\| S_t \|^2_{(C+\bar{A_t})^{\dagger}} > 2 \log \bigg( \frac{\det(\bar{C} + \bar{B}_t)^{1/2}}{\delta \det(\bar{C})^{1/2}} \bigg ) \bigg ] &= \Pr  \Bigg[ \frac{\exp \big( \frac{1}{2} \| S_t \|^2_{(C+\bar{A_t})^{\dagger}} \big) \delta }{\big( \frac{ \det(\bar{C} + \bar{B}_t)}{\det(\bar{C})} \big )^{1/2} } > 1 \Bigg] \nonumber \\
&\leq \mathbb{E}\Bigg[ \frac{\exp \big( \frac{1}{2} \| S_t \|^2_{(C+\bar{A_t})^{\dagger}} \big) \delta }{\big( \frac{ \det(\bar{C} + \bar{B}_t)}{\det(\bar{C})} \big )^{1/2} } \Bigg] \label{markov} \\
& = \mathbb{E}_{F_t}[M_t] \delta \leq \delta \label{prob_exp}
\end{align}
where \ref{markov} follows from Markov's inequality and \ref{prob_exp} is due to Lemma \ref{expected}. Notice that, $A_t = \bar{A}_t + C$ and $B_t = \bar{B}_t + \bar{C}$. We will once again use a stopping time construction. Define the bad event, 

\begin{equation*}
E_{t}(\delta) = \bigg \{ \| S_{t} \|^2_{A_{t}^{\dagger}} > 2 R^2 \log \bigg( \frac{\det(B_{t})^{1/2}}{\delta \det(\bar{C})^{1/2}} \bigg) \bigg \}.
\end{equation*}
We are interested in the probability of $\bigcup\limits_{t \geq 0} E_{t}(\delta)$. Define $\tau(\omega) = \min\{t \geq 0 : \omega \in E_{t}(\delta)\}$, with the convention that $\min \emptyset = \infty$. Then, $\tau$ is a stopping time. Thus, $\bigcup\limits_{t \geq 0} E_{t}(\delta) = \{\omega : \tau(\omega) < \infty \}$. The Theorem \ref{martin} can be obtained as follows:
\begin{align*}
  \Pr \bigg[ \bigcup\limits_{t \geq 0} E_{t}(\delta) \bigg] = \Pr[\tau < \infty] &=  \Pr \bigg [ \| S_{\tau} \|^2_{A_{\tau}^{\dagger}} > 2 R^2 \log \bigg(\frac{\det(B_{\tau})^{1/2} \det(\bar{C})^{-1/2}}{\delta }\bigg), \tau < \infty \bigg ] \\
&\leq  \Pr \bigg [ \| S_{\tau} \|^2_{A_{\tau}^{\dagger}} > 2 R^2 \log \bigg(\frac{\det(B_{\tau})^{1/2} \det(\bar{C})^{-1/2}}{\delta }\bigg), \bigg ] \leq \delta.
\end{align*}
Since $C = \lambda I_d$, inserting $\bar{C} = \lambda I_m$ proves the theorem.
\end{proof}
Combining Theorem \ref{martin} with (\ref{conf_terms}) and  Lemma \ref{error_general}, we obtain the first statement (\ref{ellips1Supple}) of Theorem \ref{mainSupple}:
\begin{equation}
\| \theta_t - \theta_* \|_{A_t} \leq R\sqrt{ 2 \log \frac{\det(B_{t})^{1/2} \det(\lambda I_m)^{-1/2}  }{\delta}} +  S \Phi_{t,\delta} \| (A_t^{\dagger})^{1/2} \hat{P}_t \hat{\Sigma}_{t-1} \|_2 + S\sqrt{\lambda} \label{first}
\end{equation}

To prove the second statement of the theorem, we need to bound $\| (A_t^{\dagger})^{1/2} \hat{P}_t \hat{\Sigma}_{t-1} \|_2$ with the help of Assumptions 1 and 2. Define $B_{t,s} = \hat{V}_t^T (\hat{\Sigma}_{s-1} + \lambda I_d) \hat{V}_t$. Note that $B_{t,t} = B_t$. Now consider the following lemmas which will be used to bound $\| (A_t^{\dagger})^{1/2} \hat{P}_t \hat{\Sigma}_{t-1} \|_2$. 

\begin{lemma} \label{detAt}
Suppose $\| \hat{x}_{t,i} \|_2 \leq L$ for all $t\geq 1$ and $i \in [K]$. Then, $\det(B_t) \leq \bigg ( \lambda + \frac{t L^2}{m} \bigg )^m$
\end{lemma}
\begin{proof}
$\det(B_t) = \det(\hat{V}_t^T \hat{\Sigma}_{t-1} \hat{V}_t  + \lambda I_m) = \alpha_1 \alpha_2 \cdots \alpha_m$ where $\alpha_i$s are the eigenvalues of $B_t$. Notice that 
\begin{equation*}
\sum_{i=1}^m \alpha_i = m \lambda + \tr\bigg( \hat{V}_t^T \big (\sum_{i=1}^{t} \hat{X}_{i-1} \hat{X}_{i-1}^T \big) \hat{V}_t  \bigg) = m \lambda + \sum_{i=1}^{t} \tr \bigg( \hat{V}_t^T \hat{X}_{i-1} \hat{X}_{i-1}^T \hat{V}_t  \bigg) \leq m\lambda + \sum_{i=1}^{t} \| \hat{X}_{i-1} \|^2_2 \leq m\lambda + tL^2.
\end{equation*}
Using AM-GM inequality, \textit{i.e}$, \sqrt[m]{\alpha_1 \alpha_2 \cdots \alpha_m} \leq \frac{1}{m} \sum_{i=1}^m \alpha_i$, we get
\begin{equation*}
\alpha_1 \alpha_2 \cdots \alpha_m \leq \bigg ( \lambda + \frac{t L^2}{m} \bigg )^m.
\end{equation*}
\end{proof}
\begin{lemma} \label{SelfNormProj} Suppose $\| \hat{x}_{t,i} \|_2 \leq L$ for all $t\geq 1$ and $i \in [K]$. Then 
\begin{equation*}
\sum_{i=1}^t \big \| \hat{V}_t^T \hat{X}_{i-1} \big \|^2_{B_{t,i-1}^{-1}} \leq \gamma m \log \bigg ( 1 + \frac{t L^2}{m \lambda} \bigg ) 
\end{equation*}
\end{lemma}
\begin{proof}
Analyzing $\det(B_t)$ at round t, we get the following:
\begin{align*}
\det\big(B_{t,t}\big) &= \det\big(B_{t,t-1} + \hat{V}_t^T \hat{X}_{t-1} \hat{X}_{t-1}^T \hat{V}_t \big) = \det\bigg(B_{t,t-1}^{1/2} \big(I_m + B_{t,t-1}^{-1/2} \hat{V}_t^T \hat{X}_{t-1} \hat{X}_{t-1}^T \hat{V}_t B_{t,t-1}^{-1/2} \big) B_{t,t-1}^{1/2} \bigg) \\
&= \det\big(B_{t,t-1}\big) \big(1 + \| \hat{V}_t^T \hat{X}_{t-1} \|^2_{B_{t,t-1}^{-1}}\big) = \lambda^m \prod_{i=1}^t \big(1 + \|\hat{V}_t^T \hat{X}_{i-1} \|^2_{B_{t,i-1}^{-1}}\big)
\end{align*}
Thus, $\sum_{i=1}^t \log (1+ \| \hat{V}_t^T \hat{X}_{i-1}  \|^2_{B_{t,i-1}^{-1}}) = \log \frac{\det(B_t)}{\lambda^m} \leq m \log \bigg ( 1 + \frac{t L^2}{m \lambda} \bigg ) $ where inequality follows from Lemma \ref{detAt}. Recall the definition of $\gamma = \frac{L^2}{\lambda \log \big(1+ \frac{L^2}{\lambda}\big ) }$. Since $\| \hat{x}_{t,i} \|_2 \leq L$ for all $t\geq 1$ and $i \in [K]$, $ \| \hat{V}_t^T \hat{X}_{i-1} \|^2_{B_{t,i-1}^{-1}} \leq \frac{L^2}{\lambda}$. Using $\| \hat{V}_t^T \hat{X}_{i-1} \|^2_{B_{t,i-1}^{-1}} \leq \gamma \log(1+\| \hat{V}_t^T \hat{X}_{i-1} \|^2_{B_{t,i-1}^{-1}})$, which is true for $\| \hat{V}_t^T \hat{X}_{i-1} \|^2_{B_{t,i-1}^{-1}}  \leq \frac{L^2}{\lambda} $,  we get
\begin{equation*}
\sum_{i=1}^t \big \| \hat{V}_t^T \hat{X}_{i-1} \big \|^2_{B_{t,i-1}^{-1}} \leq \gamma \sum_{i=1}^t \log (1+ \| \hat{V}_t^T \hat{X}_{i-1}  \|^2_{B_{t,i-1}^{-1}}) 
\end{equation*}
The lemma follows immediately.
\end{proof} 
Finally, we provide the bound on $\| (A_t^{\dagger})^{1/2} \hat{P}_t \hat{\Sigma}_{t-1} \|_2$ as follows,
\begin{lemma} \label{SecTerm} Suppose $\| \hat{x}_{t,i} \|_2 \leq L$ for all $t\geq 1$ and $i \in [K]$. Then, $ \| (A_t^{\dagger})^{1/2} \hat{P}_t \hat{\Sigma}_{t-1} \|_2 \leq L \sqrt{t} \sqrt{ \gamma m} \sqrt{ \log \bigg ( 1 + \frac{t L^2}{m \lambda} \bigg )}$. 
\end{lemma}
\begin{proof}
Recall the definition of $\hat{\Sigma}_{t-1} = \sum_{i=1}^{t-1} \hat{X}_{i} \hat{X}_{i}^T $. Using this, we get
\begin{align*}
\| (A_t^{\dagger})^{1/2} \hat{P}_t \hat{\Sigma}_{t-1} \|_2 &= \bigg \| \sum_{i=1}^{t} (A_t^{\dagger})^{1/2} \hat{P}_t \hat{X}_{i-1} \hat{X}_{i-1}^T \bigg \|_2 \\
&\leq \sum_{i=1}^t \big \| (A_t^{\dagger})^{1/2} \hat{P}_t \hat{X}_{i-1} \hat{X}_{i-1}^T \big \|_2 \quad \text{Using Weyl's inequality for singular values} \\
&\leq \sum_{i=1}^t \big \| (A_t^{\dagger})^{1/2} \hat{P}_t \hat{X}_{i-1} \big \|_2 \| \hat{X}_{i-1} \|_2 \quad \text{From Cauchy Schwarz} \\
&\leq L \sum_{i=1}^t \big \| (A_t^{\dagger})^{1/2} \hat{P}_t \hat{X}_{i-1} \big \|_2 \quad \text{From $\| \hat{x}_{t,i} \|_2 \leq L$} \\
&= L \sum_{i=1}^t \big \| \hat{P}_t \hat{X}_{i-1} \big \|_{A_t^{\dagger}} \\
&=  L \sum_{i=1}^t \big \| \hat{V}_t^T \hat{X}_{i-1} \big \|_{B_t^{-1}} \quad \text{From the equality that } \hat{X}_{i-1}^T \hat{V}_t B_t^{-1} \hat{V}_t^T \hat{X}_{i-1} = \hat{X}_{i-1}^T \hat{P}_t A_t^{\dagger} \hat{P}_t \hat{X}_{i-1} \\
&\leq L \sum_{i=1}^t \big \| \hat{V}_t^T \hat{X}_{i-1} \big \|_{B_{t,i-1}^{-1}} \quad \text{Since at round t, $B_{t,i} = B_{t,i-1} + \hat{V}_t^T \hat{X}_{i} \hat{X}_{i}^T \hat{V}_t $} \\
&\leq L \sqrt{t} \sqrt{ \sum_{i=1}^t \big \| \hat{V}_t^T \hat{X}_{i-1} \big \|^2_{B_{t,i-1}^{-1}} } \leq L \sqrt{ \gamma m t} \sqrt{ \log \bigg ( 1 + \frac{t L^2}{m \lambda} \bigg )} \quad \text{ From Lemma \ref{SelfNormProj}}
\end{align*}
\end{proof}
Now that we obtain bounds on every term at (\ref{first}), we can obtain the second statement of Theorem \ref{mainSupple} directly. For the described setting in the theorem, using Lemma \ref{detAt} and Lemma \ref{SecTerm}, we get the following
\begin{align*}
\| \theta_t - \theta_* \|_{A_t} &\leq R\sqrt{ 2 \log \bigg( \frac{1}{\delta} \bigg ) + m \log \bigg ( 1 + \frac{t L^2}{m \lambda} \bigg )  } + SL \sqrt{ \gamma m t} \sqrt{ \log \bigg ( 1 + \frac{t L^2}{m \lambda} \bigg )}  \| (P - \hat{P}_t) \|_2 + S \sqrt{\lambda} \\
&\leq R\sqrt{ 2 \log \bigg( \frac{1}{\delta} \bigg ) + m \log \bigg ( 1 + \frac{t L^2}{m \lambda} \bigg )  } + 2\Gamma S L \sqrt{\frac{\alpha}{K} \log \frac{2d}{\delta} } \sqrt{ \gamma m \log \bigg ( 1 + \frac{t L^2}{m \lambda} \bigg )}+ S \sqrt{\lambda}
\end{align*}
where the last inequality gives (\ref{ellips2}) due to Lemma \ref{errornorm}. 
\end{proof}
\section{Regret Analysis, Proof of Theorem \ref{RegretAnalysis}} \label{RegretSupple}
First consider the following lemma. 
\begin{lemma} \label{rewarddif}
At round $k$, let $\hat{x} \in D_k$. If $\nu \in C_k$, then
\begin{equation*}
|(\hat{P}_k \hat{x})^T (\nu - \theta_k) | \leq \beta_{k,\delta} \| \hat{x} \|_{A^{\dagger}_{k}}. 
\end{equation*}
\end{lemma}
\begin{proof}
\begin{align*}
|(\hat{P}_k \hat{x})^T (\nu - \theta_k) | &= | (\hat{P}_k \hat{x})^T (A_k^{\dagger})^{1/2} A^{1/2}_{k} (\nu - \theta_k)| \qquad \text{since } (A_k^{\dagger})^{1/2} A^{1/2}_{k} = \hat{P}_k \\
&= |(A_k^{\dagger})^{1/2} \hat{P}_k \hat{x})^T A^{1/2}_{k} (\nu - \theta_k)| \\
&\leq \|(A_k^{\dagger})^{1/2} \hat{P}_k \hat{x} \|_2 \| A^{1/2}_{k} (\nu - \theta_k) \|_2 \quad \text{by C.S.} \\
&\leq \beta_{k,\delta} \| \hat{P}_k \hat{x} \|_{A^{\dagger}_{k}} \qquad \text{since } \nu \in C_k. \\
&= \beta_{k,\delta} \| \hat{V}_k^T \hat{x} \|_{B^{-1}_{k}} = \beta_{k,\delta} \| \hat{x} \|_{A^{\dagger}_{k}}
\end{align*}

\end{proof}
Before providing the proof of Theorem \ref{RegretAnalysis}, consider the following lemmas:

\begin{lemma} \label{proj_mth}
For all $t\geq t_{w,\delta} $, with probability at least $1-\delta$
\begin{equation}
\lambda_m(\hat{P}_t \hat{\Sigma}_{t-1} \hat{P}_t) \geq (t-1)(\lambda_- + \sigma^2) - \sqrt{t-1} \bigg ( 4L^2 \Gamma \sqrt{\frac{\alpha}{K} \log \frac{2d}{\delta} } + \sqrt{2L(\lambda_- + \sigma^2) \log \frac{m}{\delta}} \bigg)
\end{equation}
Define $t_{r,\delta} = 1 + \bigg( \frac{ 8L^2 \Gamma \sqrt{\frac{\alpha}{K} \log \frac{2d}{\delta} } + \sqrt{8L(\lambda_- + \sigma^2) \log \frac{m}{\delta}} }{\lambda_- + \sigma^2}\bigg)^2$. Then for all $t \geq t_{r,\delta} $, with probability at least $1-\delta$,
\begin{equation}
\lambda_m(\hat{P}_t \hat{\Sigma}_{t-1} \hat{P}_t) \geq \frac{(\lambda_- + \sigma^2)}{2}(t-1).
\end{equation}
\end{lemma}
\small
\begin{proof}
\begin{align*}
\lambda_m(\hat{P}_t \hat{\Sigma}_{t-1} \hat{P}_t ) &= \lambda_m \big((\hat{P}_t - P)\hat{\Sigma}_{t-1} \hat{P}_t + P\hat{\Sigma}_{t-1}(\hat{P}_t - P) + P\hat{\Sigma}_{t-1} P\big) \\
&\geq \lambda_m (P\hat{\Sigma}_{t-1} P ) - 2(t-1) L^2 \|\hat{P}_t - P\|_2 \\
&\geq \lambda_{\min}(V^T \hat{\Sigma}_{t-1} V ) - 4L^2\Gamma \sqrt{\frac{\alpha(t-1)}{K} \log \frac{2d}{\delta} } \quad \text{from Lemma \ref{errornorm}}
\end{align*}
We also have that:
\begin{align*}
\lambda_{\max} (V^T \hat{X}_j\hat{X}_j^T V) &\leq L \quad \forall j \in \{1,\ldots, i-1 \} \\
\lambda_{\min} \bigg(\mathbb{E}\big[ \sum_{j=1}^{t-1} V^T \hat{X}_j\hat{X}_j^T V \big]  \bigg) &= (t-1)(\lambda_- + \sigma^2).
\end{align*}
Applying Theorem \ref{chernoff},
\begin{equation*}
\Pr \bigg[ \lambda_{\min}(V^T \hat{\Sigma}_t V )  \leq (t-1)(\lambda_- + \sigma^2) - \sqrt{2L (t-1)(\lambda_- + \sigma^2) \log \frac{m}{\delta}} \bigg ] \leq \delta.
\end{equation*}
Combining these with similar stopping time construction as described in previous sections we derive the first statement of lemma. Now for second statement with a constant $C$, observe that, $(t-1)(\lambda_- + \sigma^2) - \sqrt{t-1} \bigg ( 4L^2 \Gamma \sqrt{\frac{\alpha}{K} \log \frac{2d}{\delta} } + \sqrt{2L(\lambda_- + \sigma^2) \log \frac{m}{\delta}} \bigg) \geq C(t-1) $ holds if and only if $t \geq 1 + \bigg( \frac{ 4L^2 \Gamma \sqrt{\frac{\alpha}{K} \log \frac{2d}{\delta} } + \sqrt{2L(\lambda_- + \sigma^2) \log \frac{m}{\delta}} }{\lambda_- + \sigma^2 - C}\bigg)^2$. Choosing $C = \frac{\lambda_- + \sigma^2}{2}$ proves the bound.
\end{proof}
Finally, we state one more lemma which will help us derive the regret bound.
\begin{lemma} \label{basiclemma}
\begin{equation*}
2\sqrt{t+1} - 2 \leq \sum_{i=1}^t \frac{1}{\sqrt{i}} \leq 2\sqrt{t} - 1  \qquad  \qquad \log(t + 1) \leq \sum_{i=1}^t \frac{1}{i} \leq 1 + \log(t)
\end{equation*}
\end{lemma}
\begin{proof}
First one can be obtained using integral estimates and the second one is due harmonic sums.
\end{proof}
\begin{proof}[Proof of Theorem \ref{RegretAnalysis}]
The instantaneous regret, $l_i$ of the algorithm at $i$th round can be decomposed as follows:
\begin{align*}
l_i &= \hat{X}^{*T}_i \theta_* - \hat{X}^T_i \theta_* \\
&\leq (\tilde{P}_i \hat{X}_i )^T \tilde{\theta}_i - ( P \hat{X}_i)^T \theta_*  \qquad \text{ since } (\tilde{P}_i , \hat{X}_i, \tilde{\theta}_i) \text{ is optimistic}\\
&= \hat{X}_i^T (\tilde{P}_i - \hat{P}_i + \hat{P}_i) \tilde{\theta}_i - \hat{X}_i^T (\hat{P}_i + P - \hat{P}_i) \theta_*   \\
&= (\hat{P}_i \hat{X}_i )^T (\tilde{\theta}_i - \theta_i) + (\hat{P}_i \hat{X}_i )^T (\theta_i - \theta_*) + ((\hat{P}_i - P)\hat{X}_i)^T \theta_* + ((\tilde{P}_i - \hat{P}_i)\hat{X}_i)^T \tilde{\theta}_i \\
&\leq 2 \beta_{i,\delta} \| \hat{X}_i \|_{A^{\dagger}_{i}} + 2LS \|\hat{P}_i - P \|_2  \quad \text{holds $\forall i$ w.p. $1-4\delta$ due to Lemma \ref{rewarddif} and Theorem \ref{main}.} 
\end{align*}
Combining this decomposition with the fact that $l_i \leq 2$, we get 
\begin{align}
l_i &\leq 2 \min \Bigg( \beta_{i,\delta} \| \hat{X}_i \|_{A^{\dagger}_{i}} + LS  \| \hat{P}_i - P \|_2 , \quad 1 \Bigg) \nonumber \\
&\leq  2 \min (\beta_{i,\delta} \| \hat{X}_i \|_{A^{\dagger}_{i}}, 1) + 2LS \min(\| \hat{P}_i - P \|_2, 1) \nonumber \\
&\leq 2 \beta_{i,\delta} \min (\| \hat{X}_i \|_{A^{\dagger}_{i}} , 1) + 2LS \| \hat{P}_i - P \|_2 \label{inst_reg}
\end{align}
where the last inequality is due to considering the regret of the algorithm after warm-up period which provides that $\| \hat{P}_i - P \|_2 < 1$. 
Now we can provide an upper bound on the regret. For all $t \geq 1$, with probability at least $1-5\delta$, 
\begin{align}
R_t &\leq \sum_{i=1}^t  2 \beta_{i,\delta} \min (\| \hat{X}_i \|_{A^{\dagger}_{i}} , 1) + 2LS \| \hat{P}_i - P \|_2 \nonumber \\
&= 2LS \sum_{i=1}^t \| \hat{P}_i - P \|_2 + \sum_{i=1}^t 2\beta_{i,\delta} \min (\| \hat{X}_i \|_{A^{\dagger}_{i}},1) \nonumber \\
&\leq 2LS \sum_{i=1}^t  \| \hat{P}_i - P \|_2 +  2\beta_{t,\delta} \sum_{i=1}^{t} \min (\| \hat{X}_i \|_{A^{\dagger}_{i}},1) \label{betaorder} \\
&\leq 2LS \sum_{i=1}^t  \| \hat{P}_i - P \|_2 + 2\beta_{t,\delta} \sqrt{ t \sum_{i=1}^{t} \min (\| \hat{X}_i \|_{A^{\dagger}_{i}}^2 , 1) } \nonumber \\
&\leq 2LS \sum_{i=1}^t  \| \hat{P}_i - P \|_2 +2\sqrt{t}\beta_{t,\delta}  \sqrt{\sum_{i=1}^{t} \min \big(\lambda_{\max}(A^{\dagger}_{i}) L^2 , 1\big) } \label{selfnormtwonorm}\\
&\leq 2LS \sum_{i=1}^t  \| \hat{P}_i - P \|_2 +2\sqrt{t}\beta_{t,\delta}  \sqrt{\sum_{i=1}^{t} \min{ \bigg(\frac{L^2}{\lambda + \lambda_m(\hat{P}_i \hat{\Sigma}_{i-1} \hat{P}_i)} , 1\bigg) } } \label{dagger_eigs} \\
&\leq 2LS \bigg(t_{w,\delta} + 2\Gamma \sqrt{\frac{\alpha}{K} \log \frac{2d}{\delta} } \sum_{i=t_{w,\delta}}^t \frac{1}{\sqrt{i}} \bigg) \label{eiglemmainsert}  \\
&\phantom{{}=1}+2\sqrt{t}\beta_{t,\delta}  \sqrt{\sum_{i=1}^{t} \min{ \Bigg(\frac{L^2}{ \lambda + \max{ \bigg( (i-1)(\lambda_- + \sigma^2) - \sqrt{i-1} \big ( 4L^2 \Gamma \sqrt{\frac{\alpha}{K} \log \frac{2d}{\delta} } + \sqrt{2L(\lambda_- + \sigma^2) \log \frac{m}{\delta}} \big), 0\bigg)}} , 1\Bigg) }}  \notag \\
&\leq 2LS \bigg(t_{w,\delta} + 2\Gamma \sqrt{\frac{\alpha}{K} \log \frac{2d}{\delta} } \sum_{i=t_{w,\delta}}^t \frac{1}{\sqrt{i}} \bigg) + 2L\sqrt{t}\beta_{t,\delta} \sqrt{\frac{t_{r,\delta}}{\lambda} + \frac{2}{\lambda_- + \sigma^2} \sum_{i=t_{r,\delta}}^t \frac{1}{i} } \label{secondeiglemmainsert} 
\\
&\leq 2LS t_{w,\delta} + 4LS\Gamma \sqrt{\frac{\alpha}{K} \log \frac{2d}{\delta} }(2\sqrt{t} - 2\sqrt{t_{w,\delta} + 1} + 1) + 2L\sqrt{t}\beta_{t,\delta} \sqrt{\frac{t_{r,\delta}}{\lambda} + \frac{2 + 2\log t - 2\log(t_{r,\delta}+1)}{\lambda_- + \sigma^2}} \label{basiclemmainsert}
\end{align}

where (\ref{betaorder}) follows from the fact that $\beta_{1,\delta} \leq \cdots \leq \beta_{t,\delta}$, (\ref{selfnormtwonorm}) follows since $\|x\|^2_M \leq \lambda_{\max}(M) \|x\|^2_2$. Maximum eigenvalue of $A_t^{\dagger}$ is equivalent to $m$th eigenvalue of $A_t$, thus (\ref{dagger_eigs}) is obtained. Using Lemma \ref{errornorm} with Lemma \ref{proj_mth} we get (\ref{eiglemmainsert}). Using the second statement of Lemma \ref{proj_mth} gives (\ref{secondeiglemmainsert}). Finally, Lemma \ref{basiclemma} provides the bound on regret shown in (\ref{basiclemmainsert}).

Recall that $\beta_{t,\delta} = \OO\left( \Gamma \sqrt{ \frac{  \alpha m }{K} \log t } \right)$. 
Therefore, last term dominates the asymptotic upper bound on regret. Using the definition of $t_{r,\delta}$ we get that the regret of the algorithm is 
\begin{equation}
    R_t \leq \OO\left(  \frac{ \alpha \Gamma^2 \sqrt{m} }{K(\lambda_- + \sigma^2 )} \sqrt{t} \log t \right)
\end{equation}

From the definition of $\Upsilon$ and the fact that the \alg uses the confidence set of $\mathcal{C}_{t} = \mathcal{C}_{m,t} \cap \mathcal{C}_{d,t}$, the theorem follows. 
\end{proof}

\newpage

\section{Additional Experiment Results} \label{experiment_more}
\vspace{-1.8\baselineskip}
\begin{figure}[!htb]
\captionsetup{justification=centering}
\centering
\subfloat[][MNIST Regret \\ Comparison for $m\medop{=}1$]{\label{fig:regretmnistd100m1}\includegraphics[width=.3333\textwidth]{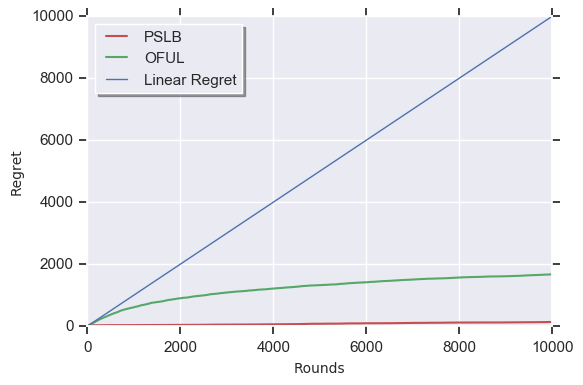}}
\centering
\subfloat[][MNIST Regret \\ Comparison for $m\medop{=}2$]{\label{fig:regretmnistd100m2}\includegraphics[width=.3333\textwidth]{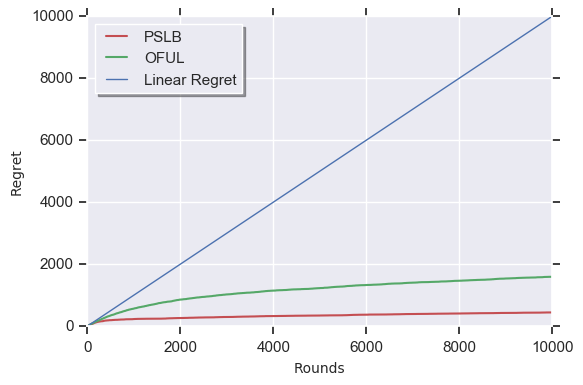}}
\centering
\subfloat[][MNIST Regret \\ Comparison for $m\medop{=}4$]{\label{fig:regretmnistd100m4}\includegraphics[width=.3333\textwidth]{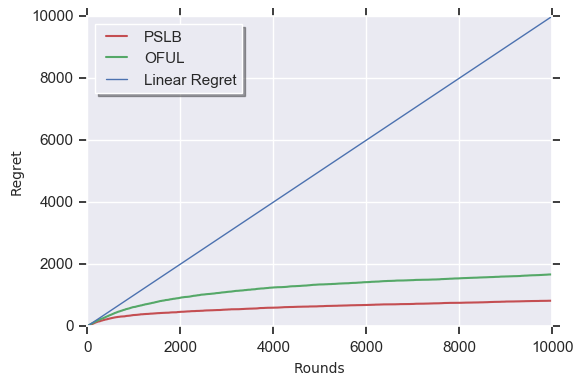}}
\vspace{-.7\baselineskip}
\centering
\subfloat[][MNIST Regret \\ Comparison for $m\medop{=}8$]{\label{fig:regretmnistd100m8}\includegraphics[width=0.3333\textwidth]{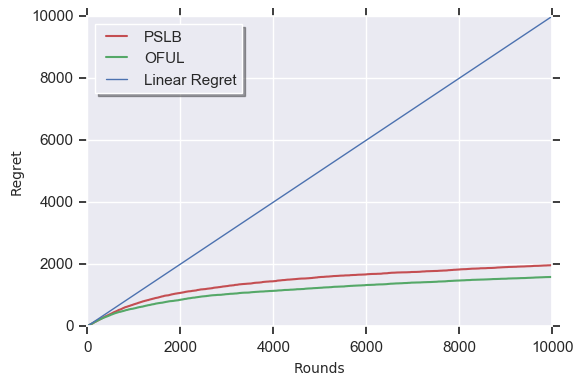}}
\subfloat[][MNIST Regret \\ Comparison for $m\medop{=}16$]{\label{fig:regretmnistd100m16}\includegraphics[width=0.3333\textwidth]{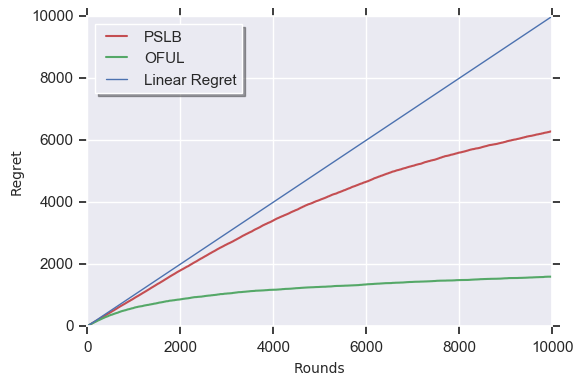}}
\vspace{-.7\baselineskip}
\centering
\subfloat[][MNIST Model Accuracy \\ Comparison for $m\medop{=}1$]{\label{fig:classmnistd100m1}\includegraphics[width=0.3333\textwidth]{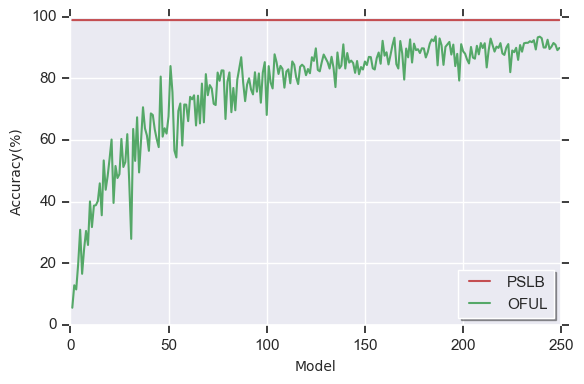}}
\centering
\subfloat[][MNIST Model Accuracy \\ Comparison for $m\medop{=}2$]{\label{fig:classmnistd100m2}\includegraphics[width=0.3333\textwidth]{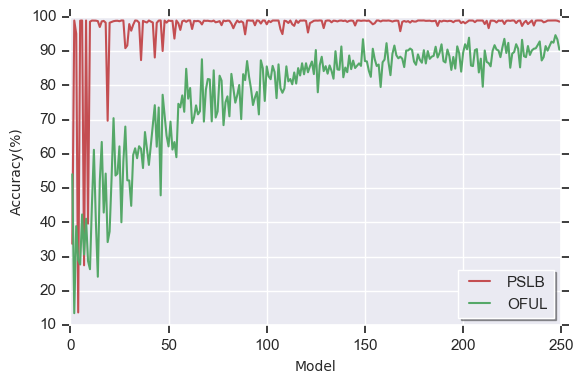}}
\centering
\subfloat[][MNIST Model Accuracy \\ Comparison for $m\medop{=}4$]{\label{fig:classmnistd100m4}\includegraphics[width=0.3333\textwidth]{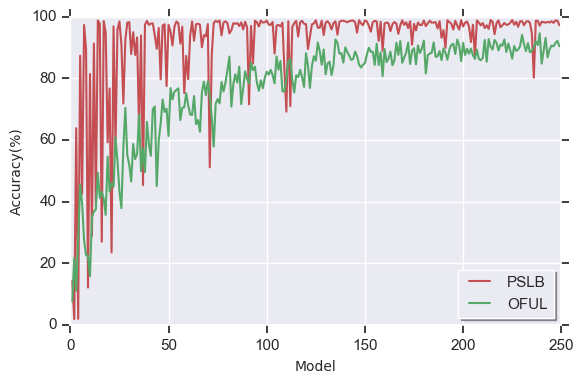}}
\vspace{-.5\baselineskip}
\centering
\subfloat[][MNIST Model Accuracy \\ Comparison for $m\medop{=}8$]{\label{fig:classmnistd100m8}\includegraphics[width=0.3333\textwidth]{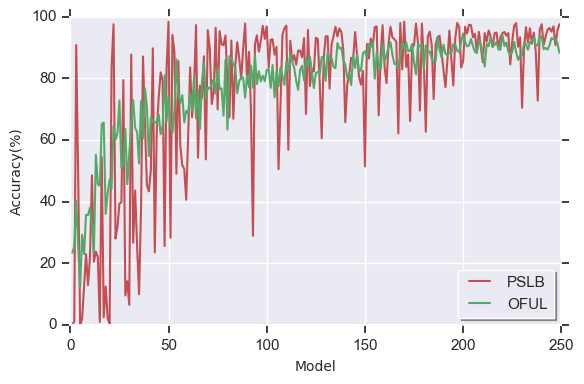}}
\subfloat[][MNIST Model Accuracy \\ Comparison for $m\medop{=}16$]{\label{fig:classmnistd100m16}\includegraphics[width=0.3333\textwidth]{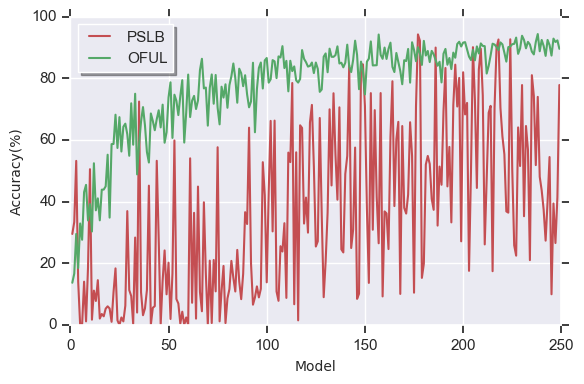}}
\vspace{-0.5\baselineskip}
\captionsetup{justification=justified }
\caption{Regret and Optimistic Model Accuracy Comparisons of \alg and \OFUL on MNIST with $d\medop{=}100$}
\label{fig:MNIST100Class}
\end{figure}
\vspace{-0.9\baselineskip}
Throughout Section \ref{experiment_more}, while running \alg, only projected confidence sets are used in choosing optimistic actions. This way, we show the effect of subspace recovery problem on the regret of \alg explicitly. Figure~\ref{fig:MNIST100Class} provides the regret and the accuracy of optimistically chosen parameters of \alg and \OFUL in \SLB constructed from MNIST with $d\medop{=}100$. Figures~\ref{fig:regretmnistd100m1}, \ref{fig:regretmnistd100m2}, \ref{fig:regretmnistd100m4}, \ref{fig:regretmnistd100m8}, \ref{fig:regretmnistd100m16} show the regrets obtained while \alg tries to recover $m=1,2,4,8,16$ dimensional subspaces respectively. Since the feature space is only $100$-dimensional \alg is not as superior over \OFUL as in high dimensional cases like $d=500,1000$. Note that as we search for a higher dimensional subspace, the subspace becomes less identifiable and finite sample projection error starts to dominate the regret. For 100-dimensional MNIST \SLB setting, when \alg tries to recover a 8-dimensional or bigger subspace, \OFUL starts to dominate \alg. Fortunately, by using the intersection of confidence sets approach, \alg tolerates this and performs at least as good as \OFUL. 

\newpage
\begin{figure}[!htb]
\captionsetup{justification=centering}
\centering
\subfloat[][MNIST Regret \\ Comparison for $m\medop{=}1$]{\label{fig:regretmnistd500m1}\includegraphics[width=.3333\textwidth]{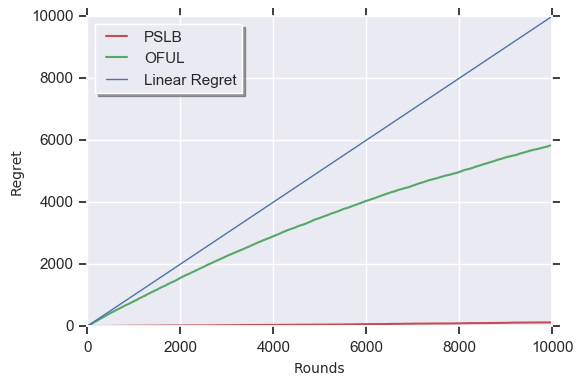}}
\centering
\subfloat[][MNIST Regret \\ Comparison for $m\medop{=}2$]{\label{fig:regretmnistd500m2}\includegraphics[width=.3333\textwidth]{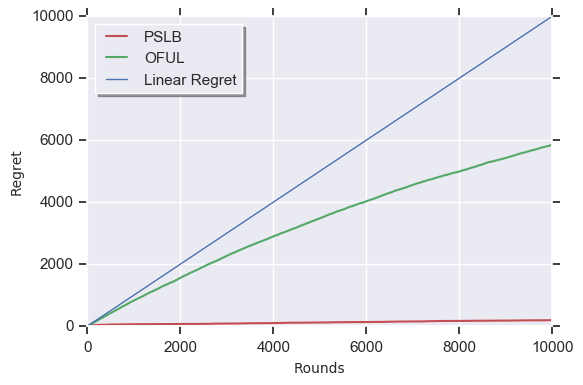}}
\centering
\subfloat[][MNIST Regret \\ Comparison for $m\medop{=}4$]{\label{fig:regretmnistd500m4}\includegraphics[width=.3333\textwidth]{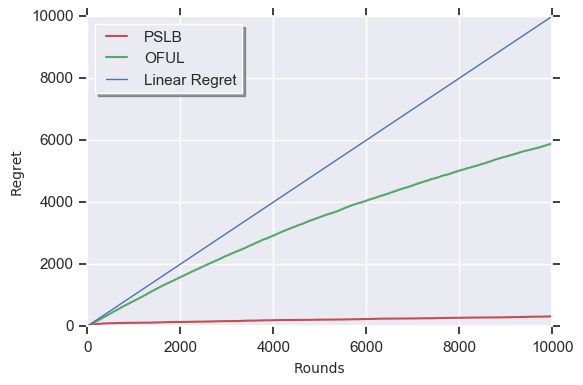}}
\vspace{-.5\baselineskip}
\centering
\subfloat[][MNIST Regret \\ Comparison for $m\medop{=}8$]{\label{fig:regretmnistd500m8}\includegraphics[width=0.3333\textwidth]{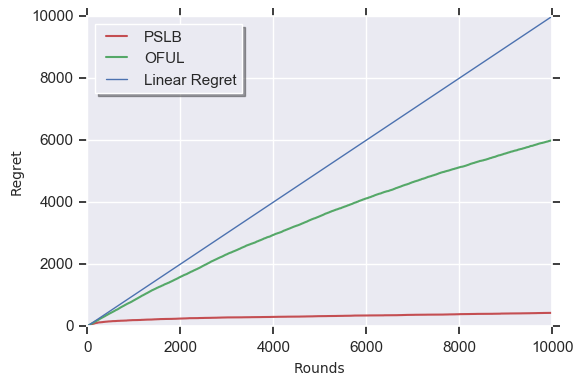}}
\subfloat[][MNIST Regret \\ Comparison for $m\medop{=}16$]{\label{fig:regretmnistd500m16}\includegraphics[width=0.3333\textwidth]{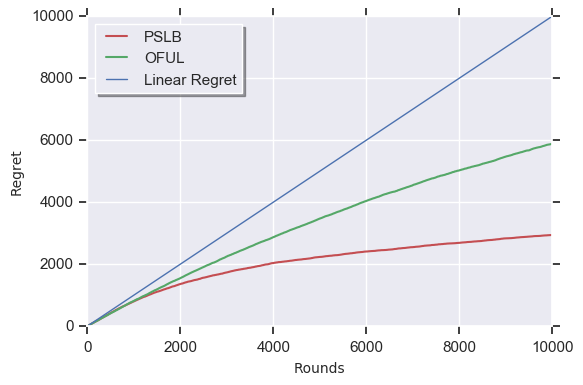}}
\vspace{-.5\baselineskip}
\centering
\subfloat[][MNIST Model Accuracy \\ Comparison for $m\medop{=}1$]{\label{fig:classmnistd500m1}\includegraphics[width=0.3333\textwidth]{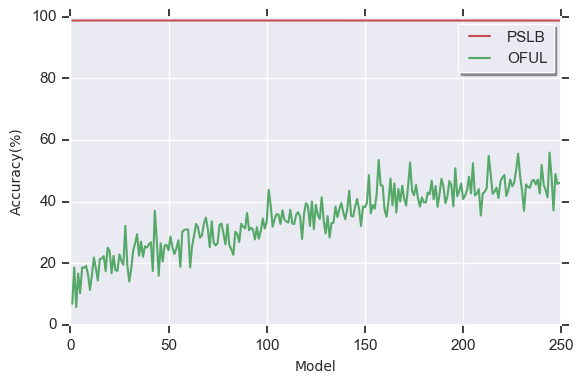}}
\centering
\subfloat[][MNIST Model Accuracy \\ Comparison for $m\medop{=}2$]{\label{fig:classmnistd500m2}\includegraphics[width=0.3333\textwidth]{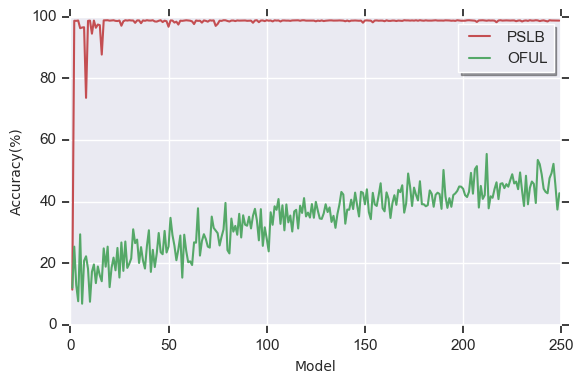}}
\centering
\subfloat[][MNIST Model Accuracy \\ Comparison for $m\medop{=}4$]{\label{fig:classmnistd500m4}\includegraphics[width=0.3333\textwidth]{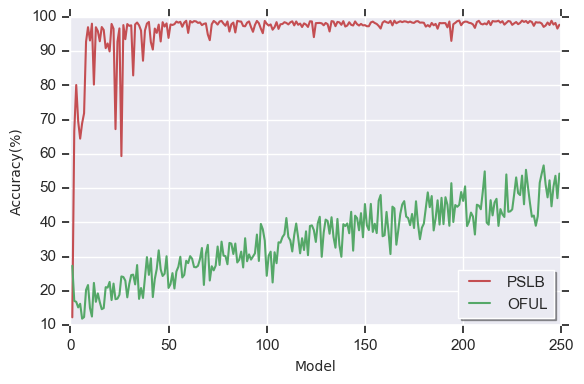}}
\vspace{-.5\baselineskip}
\centering
\subfloat[][MNIST Model Accuracy \\ Comparison for $m\medop{=}8$]{\label{fig:classmnistd500m8}\includegraphics[width=0.3333\textwidth]{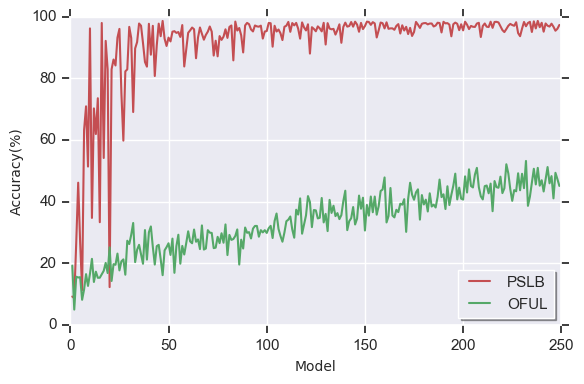}}
\subfloat[][MNIST Model Accuracy \\ Comparison for $m\medop{=}16$]{\label{fig:classmnistd500m16}\includegraphics[width=0.3333\textwidth]{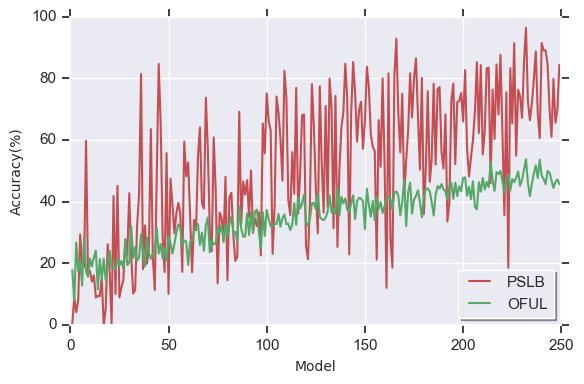}}
\vspace{-0.5\baselineskip}
\captionsetup{justification=justified }
\caption{Regret and Optimistic Model Accuracy Comparisons of \alg and \OFUL on MNIST with $d\medop{=}500$}
\label{fig:MNIST500Class}
\end{figure}
\vspace{-0.6\baselineskip}
Figure~\ref{fig:MNIST100Class} provides the regret and the accuracy of optimistically chosen parameters of \alg and \OFUL in \SLB constructed from MNIST with $d\medop{=}500$. As we go in higher dimensional representations, the benefit of subspace recovery on regret becomes more apparent. In Figure~\ref{fig:MNIST500Class}, it can be seen that \alg has smaller regret for each choice of $m$. With the PCA based subspace recovery, even for recovering higher dimensional subspaces like 8 dimensions, \alg performs well. It explores some in the beginning and as the subspace estimation gets more accurate it converges to accurate model. This behavior can be seen in Figure~\ref{fig:classmnistd500m8}.
\newpage
\begin{figure}[htb!]
\captionsetup{justification=centering}
\centering
\subfloat[][MNIST Regret \\ Comparison for $m\medop{=}1$]{\label{fig:regretmnistd1000m1}\includegraphics[width=.3333\textwidth]{experiments/experiment_new/mnist_regret_d_1000_m_1.png}}
\centering
\subfloat[][MNIST Regret \\ Comparison for $m\medop{=}2$]{\label{fig:regretmnistd1000m2}\includegraphics[width=.3333\textwidth]{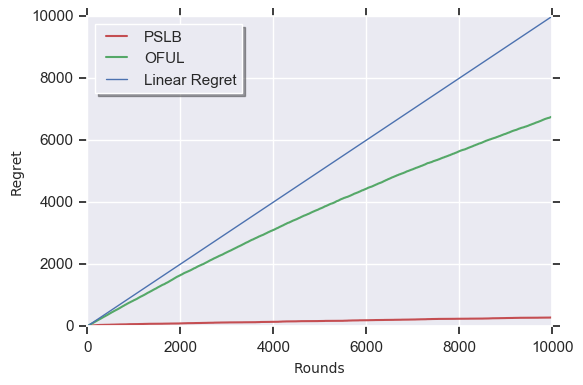}}
\centering
\subfloat[][MNIST Regret \\ Comparison for $m\medop{=}4$]{\label{fig:regretmnistd1000m4}\includegraphics[width=.3333\textwidth]{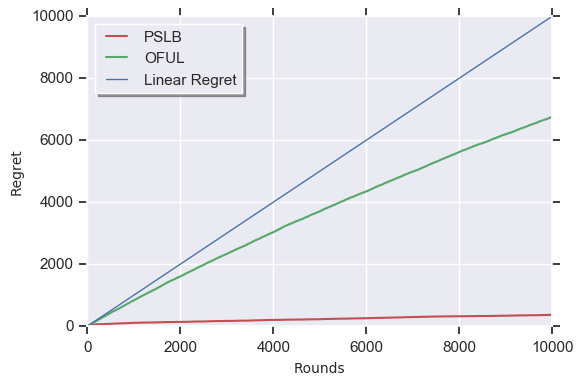}}
\vspace{-.5\baselineskip}
\centering
\subfloat[][MNIST Regret \\ Comparison for $m\medop{=}8$]{\label{fig:regretmnistd1000m8}\includegraphics[width=0.3333\textwidth]{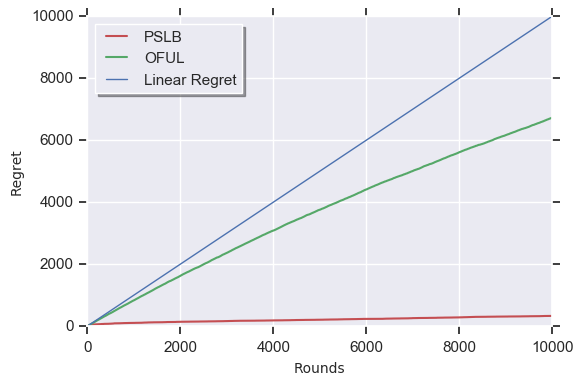}}
\subfloat[][MNIST Regret \\ Comparison for $m\medop{=}16$]{\label{fig:regretmnistd1000m16}\includegraphics[width=0.3333\textwidth]{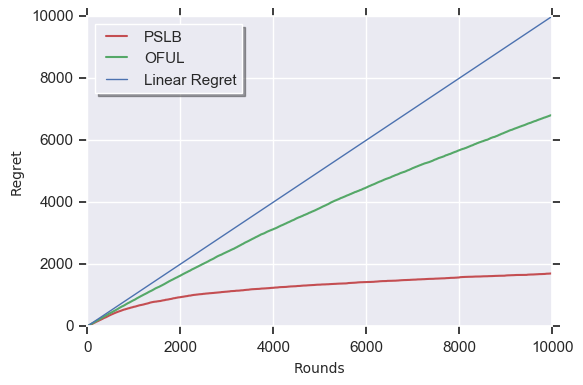}}
\vspace{-.5\baselineskip}
\centering
\subfloat[][MNIST Model Accuracy \\ Comparison for $m\medop{=}1$]{\label{fig:classmnistd1000m1}\includegraphics[width=0.3333\textwidth]{experiments/experiment_new/mnist_classification_d_1000_m_1.png}}
\centering
\subfloat[][MNIST Model Accuracy \\ Comparison for $m\medop{=}2$]{\label{fig:classmnistd1000m2}\includegraphics[width=0.3333\textwidth]{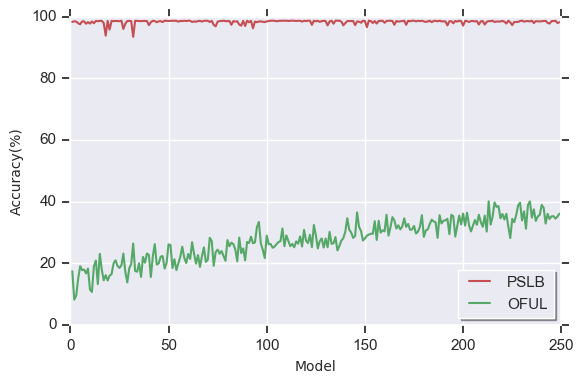}}
\centering
\subfloat[][MNIST Model Accuracy \\ Comparison for $m\medop{=}4$]{\label{fig:classmnistd1000m4}\includegraphics[width=0.3333\textwidth]{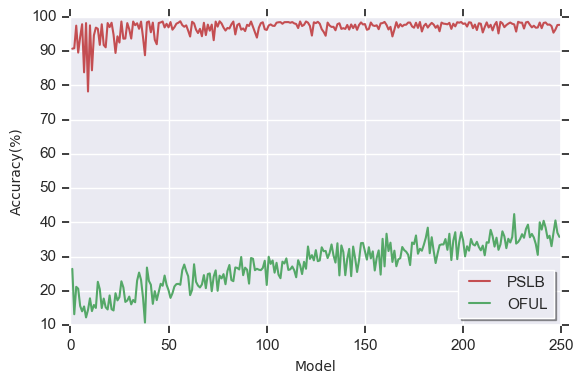}}
\vspace{-.5\baselineskip}
\centering
\subfloat[][MNIST Model Accuracy \\ Comparison for $m\medop{=}8$]{\label{fig:classmnistd1000m8}\includegraphics[width=0.3333\textwidth]{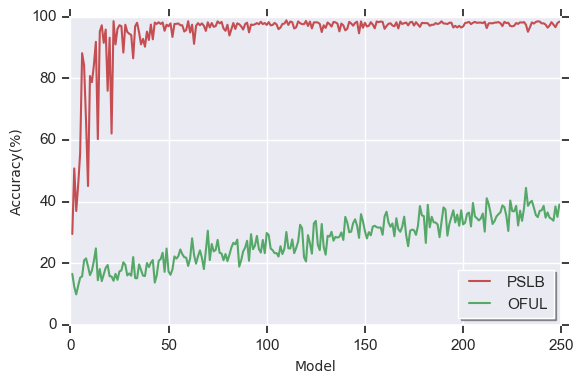}}
\subfloat[][MNIST Model Accuracy \\ Comparison for $m\medop{=}16$]{\label{fig:classmnistd1000m16}\includegraphics[width=0.3333\textwidth]{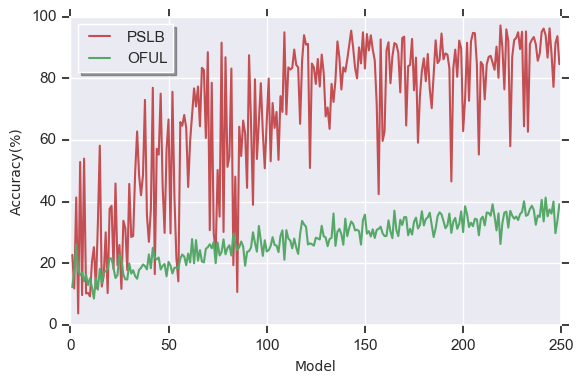}}
\vspace{-0.5\baselineskip}
\captionsetup{justification=justified }
\caption{Regret and Optimistic Model Accuracy Comparisons of \alg and \OFUL on MNIST with $d\medop{=}1000$}
\label{fig:MNIST1000Class}
\end{figure}
Figure~\ref{fig:MNIST1000Class} provides the regret and the accuracy of optimistically chosen parameters of \alg and \OFUL in \SLB constructed from MNIST with $d\medop{=}1000$. This is the setting where \alg becomes significantly superior to \OFUL. In all choices of $m$, \alg learns the underlying model accurately and starts exploiting this information. However, \OFUL still continues to explore in each dimension to figure out the underlying model. Therefore, it needs significantly more samples to achieve the classification performance of \alg and during that time it continues to make mistakes and accumulate regret. 

\newpage
\begin{figure}[!htb]
\captionsetup{justification=centering}
\centering
\subfloat[][CIFAR-10 Regret \\ Comparison for $m\medop{=}1$]{\label{fig:regretcifard100m1}\includegraphics[width=.3333\textwidth]{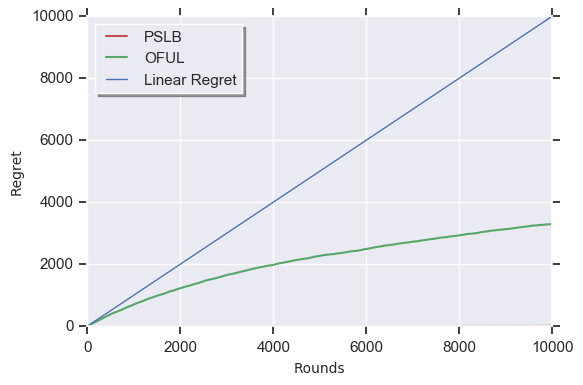}}
\centering
\subfloat[][CIFAR-10 Regret \\ Comparison for $m\medop{=}2$]{\label{fig:regretcifard100m2}\includegraphics[width=.3333\textwidth]{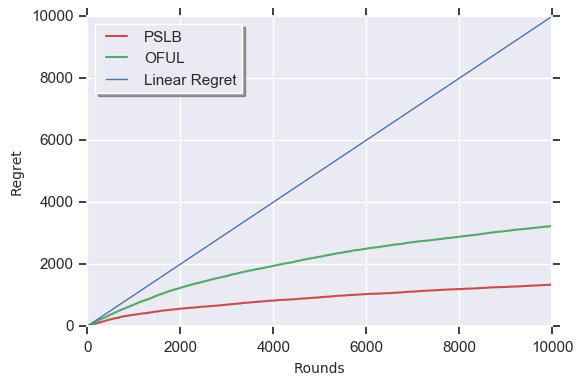}}
\centering
\subfloat[][CIFAR-10 Regret \\ Comparison for $m\medop{=}4$]{\label{fig:regretcifard100m4}\includegraphics[width=.3333\textwidth]{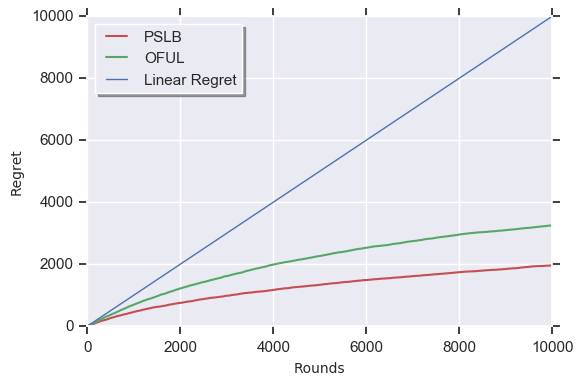}}
\vspace{-.5\baselineskip}
\centering
\subfloat[][CIFAR-10 Regret \\ Comparison for $m\medop{=}8$]{\label{fig:regretcifard100m8}\includegraphics[width=0.3333\textwidth]{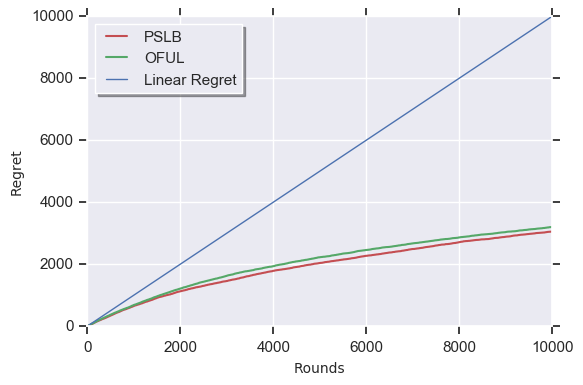}}
\subfloat[][CIFAR-10 Regret \\ Comparison for $m\medop{=}16$]{\label{fig:regretcifard100m16}\includegraphics[width=0.3333\textwidth]{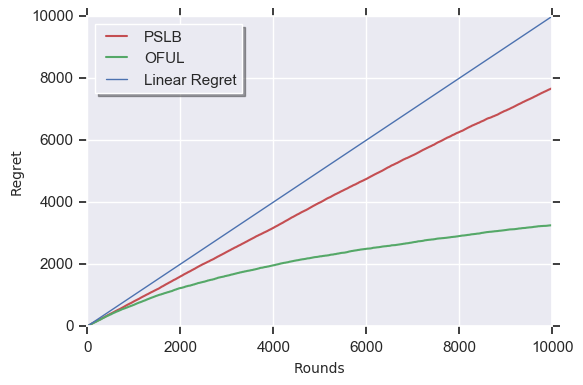}}
\vspace{-.5\baselineskip}
\centering
\subfloat[][CIFAR-10 Model Accuracy \\ Comparison for $m\medop{=}1$]{\label{fig:classcifard100m1}\includegraphics[width=0.3333\textwidth]{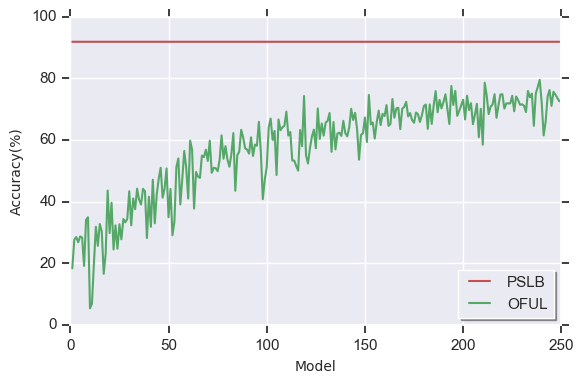}}
\centering
\subfloat[][CIFAR-10 Model Accuracy \\ Comparison for $m\medop{=}2$]{\label{fig:classcifard100m2}\includegraphics[width=0.3333\textwidth]{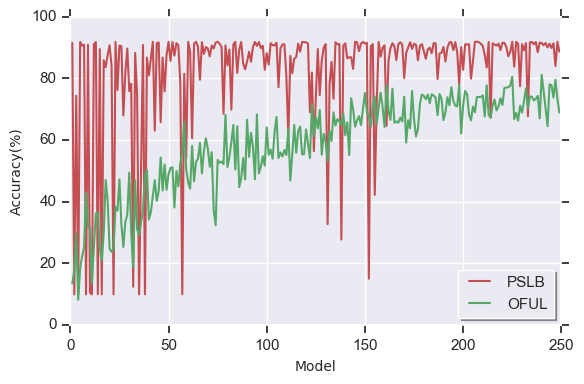}}
\centering
\subfloat[][CIFAR-10 Model Accuracy \\ Comparison for $m\medop{=}4$]{\label{fig:classcifard100m4}\includegraphics[width=0.3333\textwidth]{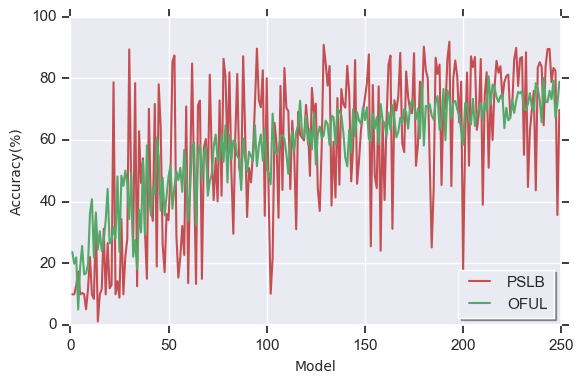}}
\vspace{-.5\baselineskip}
\centering
\subfloat[][CIFAR-10 Model Accuracy \\ Comparison for $m\medop{=}8$]{\label{fig:classcifard100m8}\includegraphics[width=0.3333\textwidth]{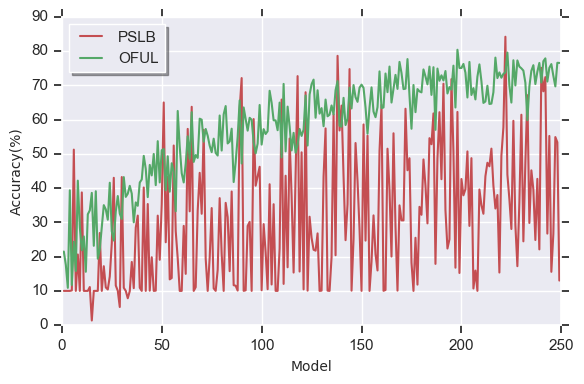}}
\subfloat[][CIFAR-10 Model Accuracy \\ Comparison for $m\medop{=}16$]{\label{fig:classcifard100m16}\includegraphics[width=0.3333\textwidth]{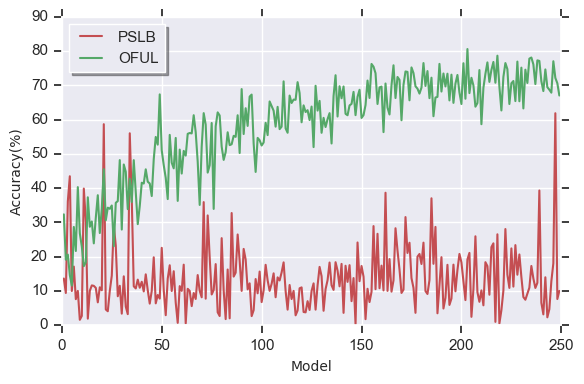}}
\vspace{-0.5\baselineskip}
\captionsetup{justification=justified }
\caption{Regret and Optimistic Model Accuracy Comparisons of \alg and \OFUL on CIFAR-10 with $d\medop{=}100$}
\label{fig:CIFAR100Class}
\end{figure}
Figure~\ref{fig:CIFAR100Class} provides the regret and the accuracy of optimistically chosen parameters of \alg and \OFUL in \SLB constructed from CIFAR-10 with $d\medop{=}100$. Similar to MNIST, due to difficulty of subspace recovery for high-dimensional subspaces, using projected confidence sets doesn't provide substantial benefit compared to \OFUL except $m=1,2$ and 4. However, the best of both algorithms approach of \alg bounds our regret with the regret of \OFUL which performs well under low dimensional ambient spaces. Moreover, we should note that by projecting the decision set onto 1-dimensional subspace, \alg makes almost no mistakes during the course of interaction, Fig~\ref{fig:regretcifard100m1}.  
\newpage
\begin{figure}[!htb]
\captionsetup{justification=centering}
\centering
\subfloat[][CIFAR-10 Regret \\ Comparison for $m\medop{=}1$]{\label{fig:regretcifard500m1}\includegraphics[width=.3333\textwidth]{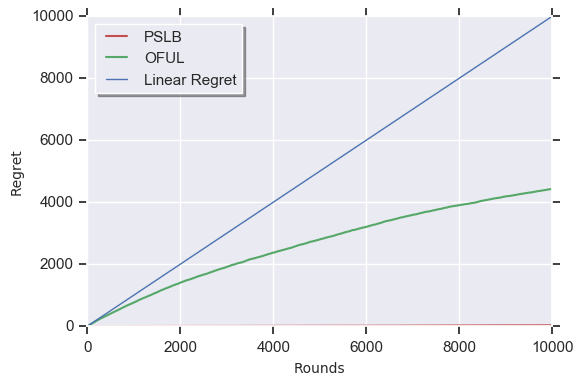}}
\centering
\subfloat[][CIFAR-10 Regret \\ Comparison for $m\medop{=}2$]{\label{fig:regretcifard500m2}\includegraphics[width=.3333\textwidth]{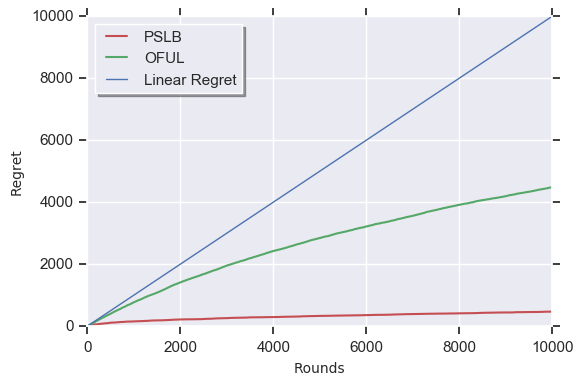}}
\centering
\subfloat[][CIFAR-10 Regret \\ Comparison for $m\medop{=}4$]{\label{fig:regretcifard500m4}\includegraphics[width=.3333\textwidth]{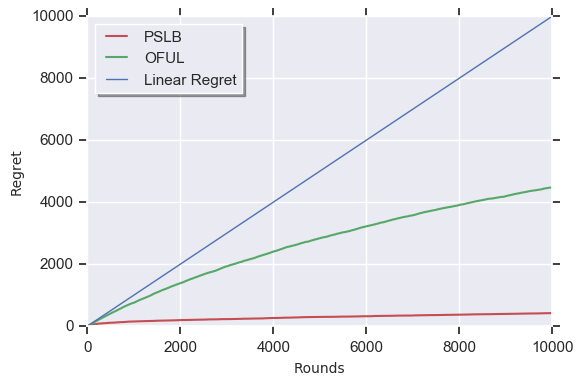}}
\vspace{-.5\baselineskip}
\centering
\subfloat[][CIFAR-10 Regret \\ Comparison for $m\medop{=}8$]{\label{fig:regretcifard500m8}\includegraphics[width=0.3333\textwidth]{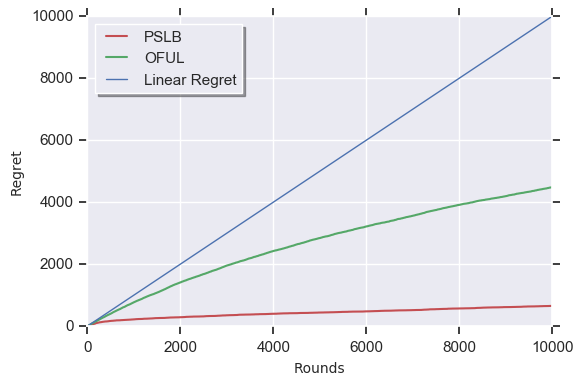}}
\subfloat[][CIFAR-10 Regret \\ Comparison for $m\medop{=}16$]{\label{fig:regretcifard500m16}\includegraphics[width=0.3333\textwidth]{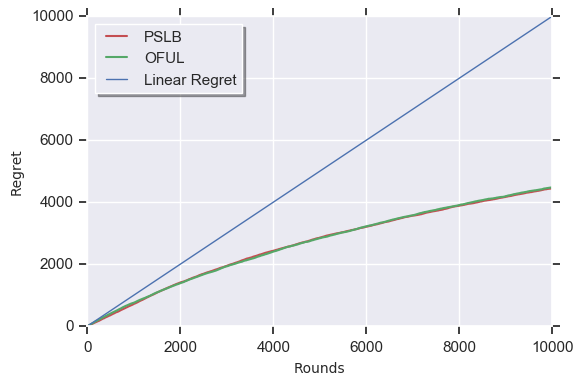}}
\vspace{-.5\baselineskip}
\centering
\subfloat[][CIFAR-10 Model Accuracy \\ Comparison for $m\medop{=}1$]{\label{fig:classcifard500m1}\includegraphics[width=0.3333\textwidth]{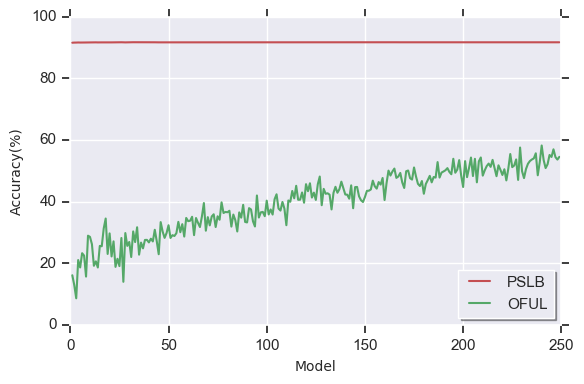}}
\centering
\subfloat[][CIFAR-10 Model Accuracy \\ Comparison for $m\medop{=}2$]{\label{fig:classcifard500m2}\includegraphics[width=0.3333\textwidth]{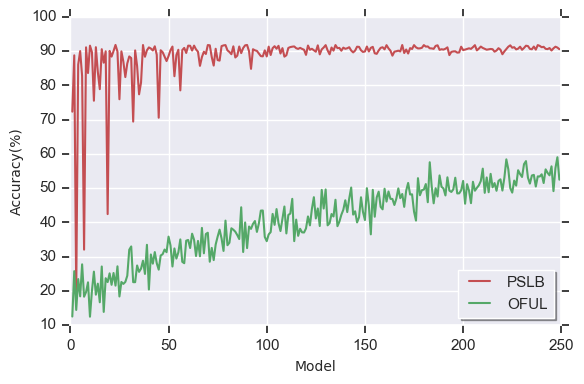}}
\centering
\subfloat[][CIFAR-10 Model Accuracy \\ Comparison for $m\medop{=}4$]{\label{fig:classcifard500m4}\includegraphics[width=0.3333\textwidth]{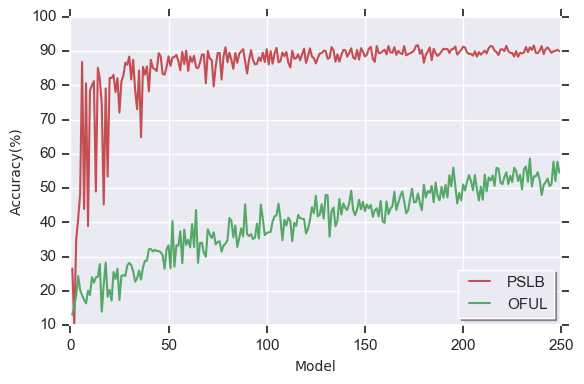}}
\vspace{-.5\baselineskip}
\centering
\subfloat[][CIFAR-10 Model Accuracy \\ Comparison for $m\medop{=}8$]{\label{fig:classcifard500m8}\includegraphics[width=0.3333\textwidth]{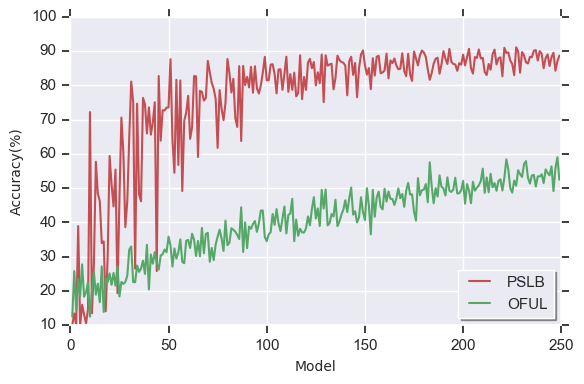}}
\subfloat[][CIFAR-10 Model Accuracy \\ Comparison for $m\medop{=}16$]{\label{fig:classcifard500m16}\includegraphics[width=0.3333\textwidth]{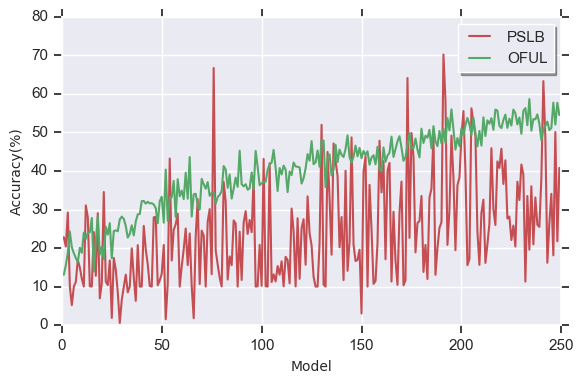}}
\vspace{-0.5\baselineskip}
\captionsetup{justification=justified }
\caption{Regret and Optimistic Model Accuracy Comparisons of \alg and \OFUL on CIFAR-10 with $d\medop{=}500$}
\label{fig:cifar500Class}
\end{figure}
Figure~\ref{fig:cifar500Class} provides the regret and the accuracy of optimistically chosen parameters of \alg and \OFUL in \SLB constructed from CIFAR-10 with $d\medop{=}500$. Similar to $d\medop{=}100$ setting, \alg makes very few mistakes when it tries to recover and project action vectors onto a 1-dimensional subspace. 
\newpage
\begin{figure}[!htb]
\captionsetup{justification=centering}
\centering
\subfloat[][CIFAR-10 Regret \\ Comparison for $m\medop{=}1$]{\label{fig:regretcifard1000m1}\includegraphics[width=.3333\textwidth]{experiments/experiment_new/cifar_regret_d_1000_m_1.png}}
\centering
\subfloat[][CIFAR-10 Regret \\ Comparison for $m\medop{=}2$]{\label{fig:regretcifard1000m2}\includegraphics[width=.3333\textwidth]{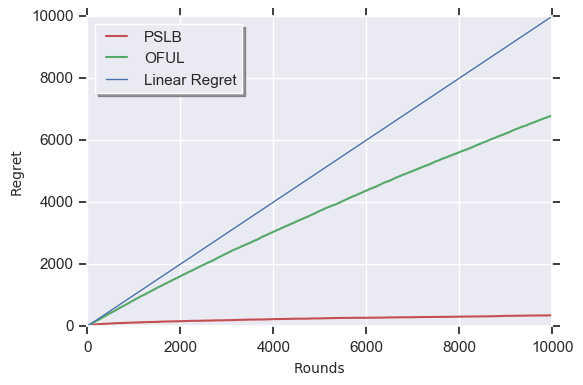}}
\centering
\subfloat[][CIFAR-10 Regret \\ Comparison for $m\medop{=}4$]{\label{fig:regretcifard1000m4}\includegraphics[width=.3333\textwidth]{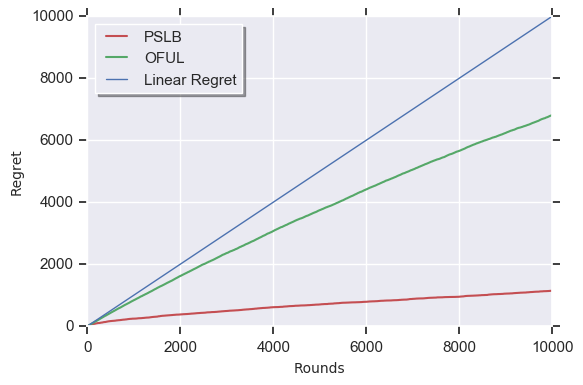}}
\vspace{-.5\baselineskip}
\centering
\subfloat[][CIFAR-10 Regret \\ Comparison for $m\medop{=}8$]{\label{fig:regretcifard1000m8}\includegraphics[width=0.3333\textwidth]{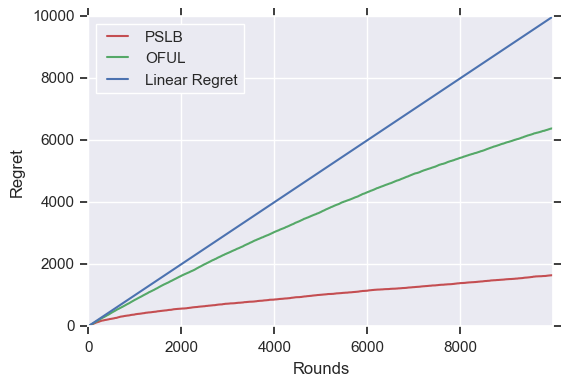}}
\subfloat[][CIFAR-10 Regret \\ Comparison for $m\medop{=}16$]{\label{fig:regretcifard1000m16}\includegraphics[width=0.3333\textwidth]{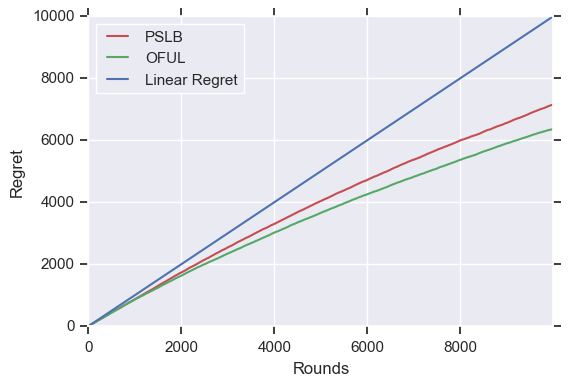}}
\vspace{-.5\baselineskip}
\centering
\subfloat[][CIFAR-10 Model Accuracy \\ Comparison for $m\medop{=}1$]{\label{fig:classcifard1000m1}\includegraphics[width=0.3333\textwidth]{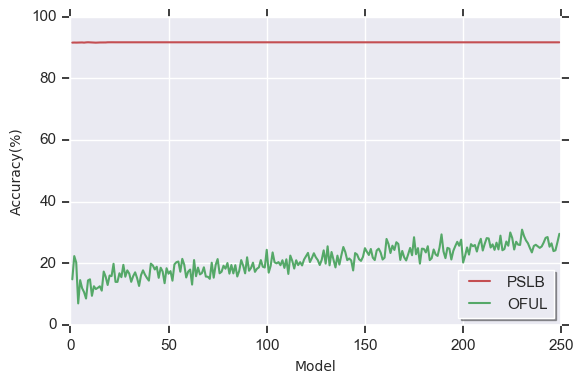}}
\centering
\subfloat[][CIFAR-10 Model Accuracy \\ Comparison for $m\medop{=}2$]{\label{fig:classcifard1000m2}\includegraphics[width=0.3333\textwidth]{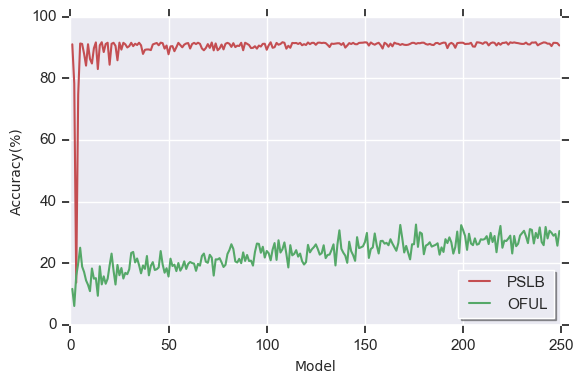}}
\centering
\subfloat[][CIFAR-10 Model Accuracy \\ Comparison for $m\medop{=}4$]{\label{fig:classcifard1000m4}\includegraphics[width=0.3333\textwidth]{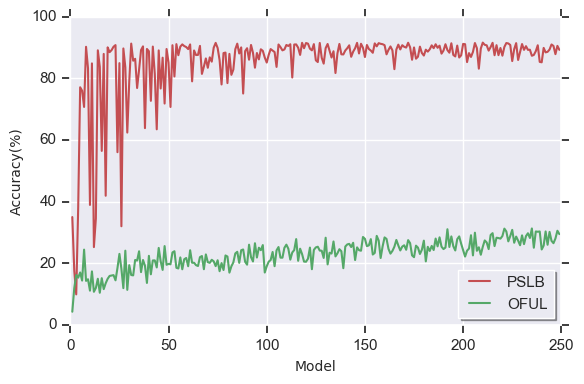}}
\vspace{-.5\baselineskip}
\centering
\subfloat[][CIFAR-10 Model Accuracy \\ Comparison for $m\medop{=}8$]{\label{fig:classcifard1000m8}\includegraphics[width=0.3333\textwidth]{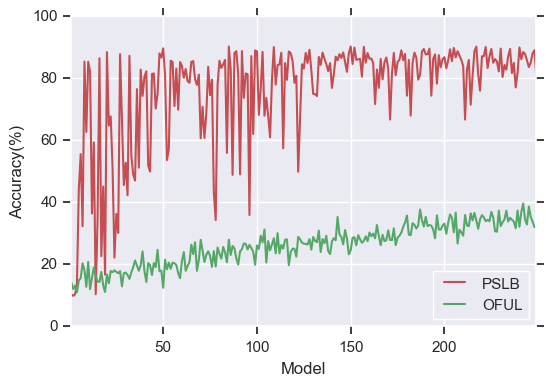}}
\subfloat[][CIFAR-10 Model Accuracy \\ Comparison for $m\medop{=}16$]{\label{fig:classcifard1000m16}\includegraphics[width=0.3333\textwidth]{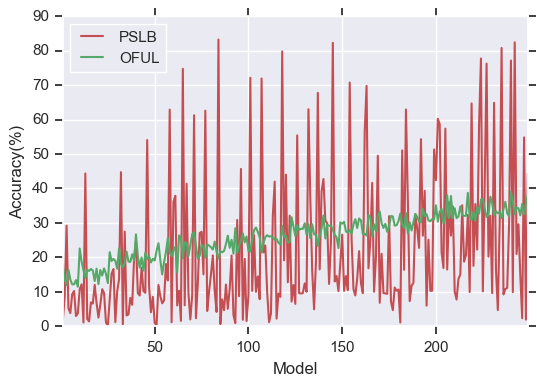}}
\vspace{-0.5\baselineskip}
\captionsetup{justification=justified }
\caption{Regret and Optimistic Model Accuracy Comparisons of \alg and \OFUL on CIFAR-10 with $d\medop{=}1000$}
\label{fig:cifar1000Class}
\end{figure}
Figure~\ref{fig:cifar1000Class} provides the regret and the accuracy of optimistically chosen parameters of \alg and \OFUL in \SLB constructed from CIFAR-10 with $d\medop{=}1000$. 
\newpage
\begin{figure}[!htb]
\captionsetup{justification=centering}\label{fig:imagenet100Regret}
\centering
\subfloat[][ImageNet Regret \\ Comparison for $m\medop{=}1$]{\label{fig:regretimagenetd100m1}\includegraphics[width=.3333\textwidth]{experiments/imagenet_regret_d_100_m_1.png}}
\centering
\subfloat[][ImageNet Regret \\ Comparison for $m\medop{=}2$]{\label{fig:regretimagenetd100m2}\includegraphics[width=.3333\textwidth]{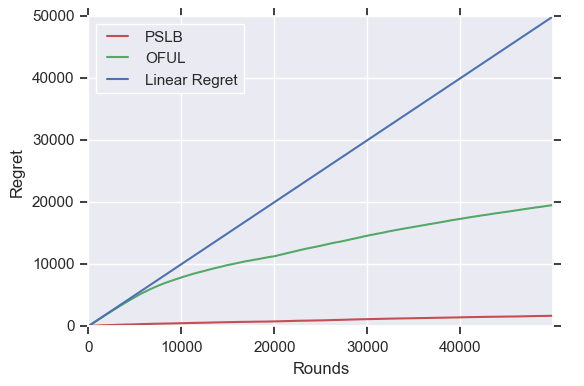}}
\centering
\subfloat[][ImageNet Regret \\ Comparison for $m\medop{=}4$]{\label{fig:regretimagenetd100m4}\includegraphics[width=.3333\textwidth]{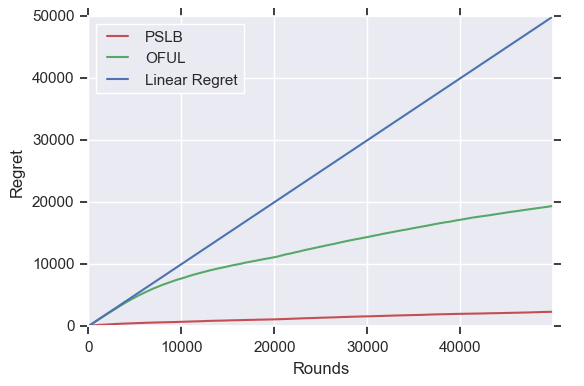}}
\vspace{-.5\baselineskip}
\centering
\subfloat[][ImageNet Regret \\ Comparison for $m\medop{=}8$]{\label{fig:regretimagenetd100m8}\includegraphics[width=0.3333\textwidth]{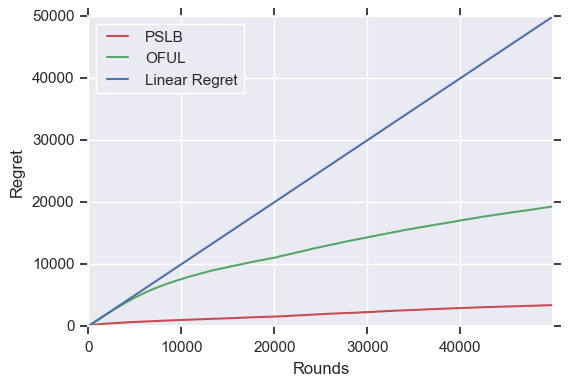}}
\subfloat[][ImageNet Regret \\ Comparison for $m\medop{=}16$]{\label{fig:regretimagenetd100m16}\includegraphics[width=0.3333\textwidth]{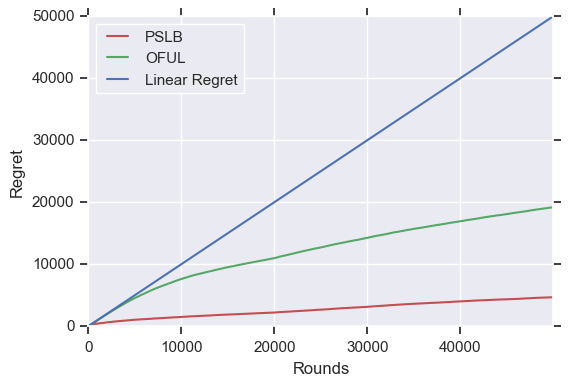}}
\vspace{-.5\baselineskip}
\centering
\subfloat[][ImageNet Model Accuracy \\ Comparison for $m\medop{=}1$]{\label{fig:classimagenetd100m1}\includegraphics[width=0.3333\textwidth]{experiments/imagenet_classification_d_100_m_1.png}}
\centering
\subfloat[][ImageNet Model Accuracy \\ Comparison for $m\medop{=}2$]{\label{fig:classimagenetd100m2}\includegraphics[width=0.3333\textwidth]{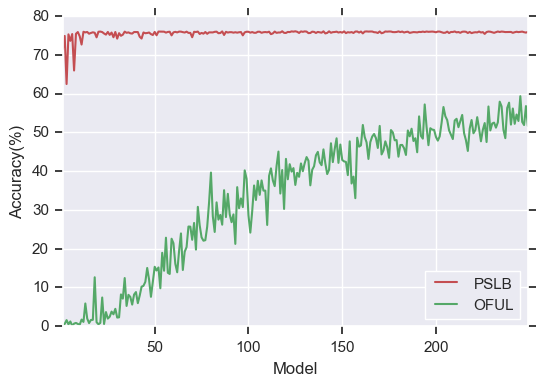}}
\centering
\subfloat[][ImageNet Model Accuracy \\ Comparison for $m\medop{=}4$]{\label{fig:classimagenetd100m4}\includegraphics[width=0.3333\textwidth]{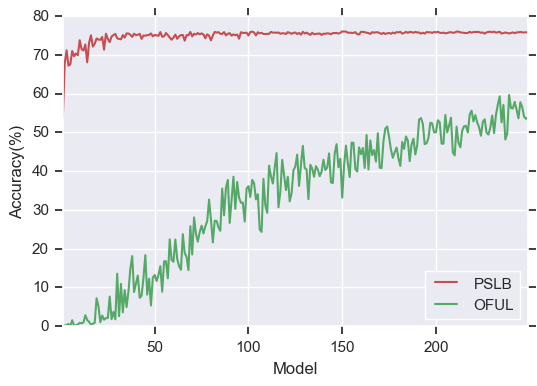}}
\vspace{-.5\baselineskip}
\centering
\subfloat[][ImageNet Model Accuracy \\ Comparison for $m\medop{=}8$]{\label{fig:classimagenetd100m8}\includegraphics[width=0.3333\textwidth]{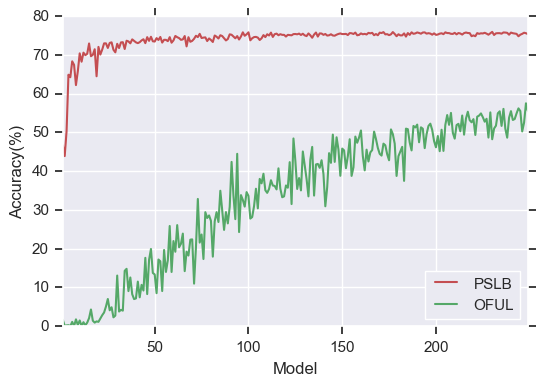}}
\subfloat[][ImageNet Model Accuracy \\ Comparison for $m\medop{=}16$]{\label{fig:classimagenetd100m16}\includegraphics[width=0.3333\textwidth]{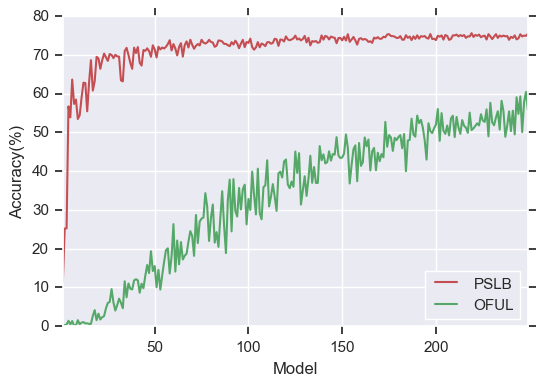}}
\vspace{-0.5\baselineskip}
\captionsetup{justification=justified }
\caption{Regret and Optimistic Model Accuracy Comparisons of \alg and \OFUL On ImageNet with $d\medop{=}100$}
\label{fig:imagenet100Class}
\end{figure}
\vspace{-1.3\baselineskip}
Figure~\ref{fig:imagenet100Class} provides the regret and the accuracy of optimistically chosen parameters of \alg and \OFUL in \SLB constructed from ImageNet with $d\medop{=}100$. Since there are 1000 different classes in the dataset, \SLB framework synthesized from ImageNet dataset has 1000 actions in each decision set. Therefore, even if $d{=}100$ is not a fairly high dimensional feature space, having 1000 actions makes the learning task harder. Thus, \SLB algorithms are expected to have higher regrets and slower convergence to underlying model. However, large number of actions is key to having lower regret for \alg. Instead of ignoring actions that are not chosen at the current round, \alg uses them to get an idea about the structure of the action vectors. This setting clearly points out the advantage of \alg over \OFUL. While \OFUL obtains linear regret in the beginning and struggles to construct a meaningful confidence set, \alg uses hidden information in the massive number of action vectors and reduce the dimensionality of the \SLB framework. Then it exploits this information and converges to the accurate model without committing many mistakes.

\end{document}